\documentclass{article}
\usepackage{arxiv}
\usepackage{authblk}

\usepackage{fullpage} %
\usepackage{multirow}
\usepackage[utf8]{inputenc}
\usepackage[english]{babel}
\usepackage{amssymb,amsfonts}
\usepackage{amsthm}
\usepackage{float}
\usepackage{booktabs}
\usepackage{color}
\usepackage{listings}
\usepackage{rotating}
\usepackage[titletoc]{appendix}
\usepackage{pifont}
\usepackage{stmaryrd}
\usepackage{array}
\usepackage[algoruled]{algorithm2e}

\usepackage{graphicx}
\usepackage{multicol}
\usepackage{enumitem}
\usepackage{framed}
\usepackage[
	operators,
	sets,
	landau,
	probability,
	notions,
	logic,
	ff,
	mm,
	advantage,
	events,
	complexity,
	asymptotics,
	keys,
	lambda]{cryptocode}
\usepackage{comment}
\usepackage{adjustbox}
\usepackage{varioref}
\usepackage{hyperref}
\usepackage{breakcites}
\usepackage[section]{placeins}
\usepackage{orcidlink}
\usepackage{subcaption}
\usepackage{tcolorbox}

\usepackage{tikz}
\usetikzlibrary{arrows.meta,decorations.pathmorphing,decorations.footprints,fadings,calc,trees,mindmap,shadows,decorations.text,patterns,positioning,shapes,matrix,fit}

\usepackage{url} %
\usepackage{framed}
\usepackage{comment}
\usepackage{adjustbox}
\usepackage[nameinlink,capitalize]{cleveref}

\floatstyle{boxed}
\definecolor{cadmiumorange}{rgb}{0.93, 0.53, 0.18}
\definecolor{emerald}{rgb}{0.31, 0.78, 0.47}
\definecolor{amaranth}{rgb}{0.9, 0.17, 0.31}
\definecolor{candypink}{rgb}{0.89, 0.44, 0.48}
\definecolor{caribbeangreen}{rgb}{0.0, 0.8, 0.6}
\definecolor{cornflowerblue}{rgb}{0.39, 0.58, 0.93}
\definecolor{mayablue}{rgb}{0.21,0.49,0.74}
\definecolor{limegreen}{rgb}{0.2, 0.8, 0.2}
\usepackage{utfsym}

\newcommand{\wrong}{\textcolor{red}{\usym{2717}}}
\newcommand{\coloright}{\textcolor{green}{\usym{2713}}}
\usepackage{xcolor}

\usepackage{pifont}
\usepackage{tocloft}
\usepackage{adjustbox}
\usepackage{minitoc}
\hypersetup{
  linkcolor  = amaranth,
    filecolor=magenta,      
    urlcolor=teal,
    citecolor=mayablue,
    colorlinks = true,
    breaklinks = true,
    bookmarksopen=true
}

\usepackage{epigraph}

\usepackage{subcaption} %

\newcommand*\samethanks[1][\value{footnote}]{\footnotemark[#1]}

\begin{document}

\addtocontents{toc}{\protect\setcounter{tocdepth}{0}}

\title{Watermarks in the Sand: Impossibility of Strong Watermarking for Generative Models}

\author[1]{Hanlin Zhang}
\author[1]{Benjamin L. Edelman\thanks{Equal contribution.}}
\author[2]{Danilo Francati\samethanks[1]}
\author[3]{\\ Daniele Venturi} 
\author[2]{{Giuseppe Ateniese}}
\author[1]{Boaz Barak}

\affil[1]{Harvard University, Cambridge, MA, USA}
\affil[2]{George Mason University, Fairfax, VA, USA}
\affil[3]{Sapienza University of Rome, Rome, Italy}

\date{}
\doparttoc
\faketableofcontents

\maketitle

\newcommand{\TODO}[1]{\textcolor{red}{\textbf{TODO:} #1}}
\newcommand{\DFnote}[1]{\textcolor{cyan}{\textbf{Danilo:} #1}}
\newcommand{\GAnote}[1]{\textcolor{red}{\textbf{Giuseppe:} #1}}
\newcommand{\BBnote}[1]{\textcolor{green}{\textbf{Boaz:} #1}}
\newcommand{\DVnote}[1]{\textcolor{blue}{\textbf{Daniele:} #1}}
\newcommand{\hanlin}[1]{\textcolor{orange}{\textbf{Hanlin:} #1}}
\newcommand{\BEnote}[1]{\textcolor{brown}{\textbf{Ben:} #1}}

\newtheorem{theorem}{Theorem}

\theoremstyle{definition}
\newtheorem{definition}{Definition}
\newtheorem{corollary}{Corollary}
\newtheorem{lemma}{Lemma}
\newtheorem{remark}{Remark}

\crefname{construction}{construction}{constructions}
\Crefname{construction}{Construction}{Constructions}
\Crefname{claim}{Claim}{Claims}
\Crefformat{claim}{#2Claim~(#1)#3}
\Crefrangeformat{claim}{Claims~(#3#1#4) to~(#5#2#6)}

\newcommand{\Prob}[1]{\prob{#1}}
\renewcommand{\bigO}{O}
\newcommand{\getsr}{\overset{r}{\gets}}
\renewcommand{\minentropy}{\mathbb{H}_{\infty}}

\renewcommand{\set}[1]{\mathcal{#1}}
\newcommand{\cX}{\set{X}}
\newcommand{\cY}{\set{Y}}
\newcommand{\cK}{\set{K}}
\newcommand{\cM}{\set{M}}
\newcommand{\cQ}{\set{Q}}
\newcommand{\cQuery}{\cQ}
\newcommand{\cR}{\set{R}}
\newcommand{\cT}{\set{T}}
\newcommand{\cP}{\set{P}}
\newcommand{\cV}{\mathcal{V}}
\newcommand{\cE}{\mathcal{E}}

\newcommand{\classModel}{\{\model_i:\cX\rightarrow\cY\}}

\newcommand{\rv}[1]{\mathbf{#1}}
\newcommand{\rvX}{\rv{X}}
\newcommand{\rvY}{\rv{Y}}
\newcommand{\rvG}{\rv{G}}
\newcommand{\game}{\rv{G}}
\newcommand{\rvM}{\rv{M}}

\newcommand{\alg}[1]{\mathsf{#1}}
\newcommand{\model}{\alg{M}}
\newcommand{\wmodel}{\hat{\model}}
\newcommand{\watermark}{\alg{Watermark}}
\newcommand{\detect}{\alg{Detect}}
\newcommand{\responseModel}{\overline{\alg{M}}}
\newcommand{\quality}{\mathsf{Q}}
\renewcommand{\kgen}{\mathsf{KGen}}

\newcommand{\adversary}{\alg{A}}
\newcommand{\adv}{\adversary}
\newcommand{\distinguisher}{\alg{D}}
\newcommand{\dist}{\distinguisher}
\newcommand{\simulator}{\alg{S}}
\renewcommand{\sim}{\simulator}
\newcommand{\extractor}{\alg{E}}
\newcommand{\ext}{\extractor}
\newcommand{\oracle}{\mathsf{O}}
\newcommand{\perturbationOracle}{\mathsf{P}}
\newcommand{\pO}{\perturbationOracle}

\renewcommand{\deg}{\mathsf{deg}}

\newcommand{\msg}{m}
\newcommand{\cipher}{c}
\renewcommand{\k}{\mathsf{k}}
\newcommand{\prompt}{p}
\newcommand{\done}{\mathsf{done}}
\newcommand{\graph}{{\mathsf{G}}}
\newcommand{\graphTuple}{\graph = (\cV,\cE)}
\renewcommand{\matrix}[1]{\vec{#1}}
\renewcommand{\vector}[1]{\vec{#1}}
\newcommand{\eigenvalue}{\alpha}
\newcommand{\ctr}{\mathsf{ctr}}
\newcommand{\weight}{\mathsf{weight}}

\newcommand{\R}{\mathbb{R}}

\newcommand{\epos}{\epsilon_{\text{\sf pos}}}

\newcommand{\eneg}{\epsilon_{\text{\sf neg}}}

\newcommand{\edist}{\epsilon_{\text{\sf dist}}}

\newcommand{\citeabstract}[1]{\citeauthor{#1}~(\citeyear{#1})}

\begin{abstract}
\emph{Watermarking} generative models consists of planting a statistical signal (watermark) in a model's output so that it can be later verified that the output was generated by the given model.
A \emph{strong} watermarking scheme satisfies the property that a computationally bounded attacker cannot erase the watermark without causing significant quality degradation.
In this paper, we study the (im)possibility of strong watermarking schemes. 
We prove that, under well-specified and natural assumptions, {\em strong} watermarking is \emph{impossible to achieve}.
This holds even in the \emph{private detection algorithm} setting, where the watermark insertion and detection algorithms share a secret key, unknown to the attacker.
To prove this result, we introduce a \emph{ generic} efficient watermark attack; the attacker is not required to know the private key of the scheme or even which scheme is used.

Our attack is based on two assumptions: (1) The attacker has access to a ``quality oracle'' that can evaluate whether a candidate output is a high-quality response to a prompt, and (2) The attacker has access to a ``perturbation oracle'' which can modify an output with a nontrivial probability of maintaining quality, and which induces an efficiently mixing random walk on high-quality outputs.
We argue that both assumptions can be satisfied in practice by an attacker with weaker computational capabilities than the watermarked model itself, to which the attacker has only black-box access.
Furthermore, our assumptions will likely only be easier to satisfy over time as models grow in capabilities and modalities. 

We demonstrate the feasibility of our attack by instantiating it to attack three existing watermarking schemes for large language models: \citeabstract{kirchenbauer2023watermark}, \citeabstract{kuditipudi2023robust}, and \citeabstract{zhao2023provable}, as well as those for vision-language models \citeabstract{fernandez2023stable} and \citeabstract{shield2023ivw}. 
The same attack successfully removes the watermarks planted by all schemes, with only minor quality degradation.\footnote{Project website and demo at {\href{https://hanlin-zhang.com/impossibility-watermarks/}{\textcolor{mayablue}{\texttt{https://hanlin-zhang.com/impossibility-watermarks}}}}.}

\end{abstract}

\section{Introduction}\label{sec:intro}

The advent of powerful generative models such as large language models (LLMs) \citep{brown2020language, schulman2022chatgpt, openai2023gpt4} and text-to-image models \citep{radford2021learning, openai2023dalle3} has ushered in a new era where machines can be prompted to answer questions, draft documents, generate images in various styles, write executable code, and more. 
As these models become increasingly widely deployed and capable, there is growing concern that malicious actors could misrepresent model outputs as human-generated content \citep{clark2021all}. To prevent misuse at scale---e.g., misinformation \citep{nytimes2023disinfo}, automated phishing \citep{hazell2023large}, and academic cheating \citep{kasneci2023chatgpt}---there has been a demand for algorithmic methods that can distinguish between content produced by models and by humans \citep{westerlund2019emergence}. 
Some solutions \citep{mitchell2023detectgpt, tian2023gptzero} aim to detect the output of a given model or family of models without modifying the model's generation process at all, but these tend to suffer from high error rates \citep{kirchenbauer2023reliability}. The idea of \emph{watermarking schemes} for generative models is to alter the inference procedure to plant identifiable statistical signals (watermarks) into the model outputs \citep{kirchenbauer2023watermark, kuditipudi2023robust,zhao2023provable, kirchenbauer2023reliability,christ2023undetectable, fernandez2023stable, aaronson2022ea}. Because these schemes can intervene in the generation process, they are potentially more powerful than classical digital watermarking schemes \citep{o1996watermarking, o1997rotation}, which add imperceptible watermarks to individual given pieces of content. 

A watermarking scheme consists of a \emph{generation} algorithm that is a modified version of the model in which the signal is planted and a \emph{detection} algorithm that can detect whether a piece of output came from the watermarked model.
Watermarking schemes can be partitioned onto at least two different axes:
\begin{itemize}
\item\textbf{Public versus private:} A \textbf{public} watermarking scheme is one where the detection algorithm is accessible to all parties.
A \textbf{private} (or \emph{secret-key}) watermarking scheme is one in which running the detection algorithm requires some private information.
\item\textbf{Strong versus weak:} A \textbf{strong} watermarking scheme is one where a (computationally bounded) attacker cannot modify the output (e.g., rephrase the text, apply a filter to the image, etc.) to remove the watermark without causing significant quality degradation. 
A \textbf{weak} watermarking scheme only resists removal by a well-specified set of transformations.
At a minimum, the detection algorithm of a weak scheme must flag outputs that are simply ``copied and pasted''. More generally, it can ensure the detection of modified outputs as long as the modification is close to the original according to some metric such as edit distance for text, $\ell_1$ norm for images. 
Weak watermarks can still be useful for applications like preventing AI-generated content from being used for training \citep{shumailov2023curse}, making it more expensive or inconvenient to generate misinformation or cheat on assignments, and tracking the provenance of the precise text/image/etc, which the watermark was applied to.
But they will not foil a determined attacker.
\end{itemize}

In this paper, we focus on \emph{strong} watermarking and hence drop the ``strong'' modifier from this point on.
(See \Cref{sec:watermarking} for the formal definition.)
Our main result is negative: Under mild assumptions (which we specify below), strong watermarking of generative models is \emph{impossible}.
This holds even in the more challenging (for the attacker) \emph{secret-key} watermarking setting, where the adversary cannot access the watermarking algorithm.

Our assumptions already hold today in several settings, and we argue that they will become only more likely as models grow in both capabilities and modalities. 
The impossibility result is constructive: we design a generic attack methodology that can remove any watermark, given the assumptions.
Our attacker algorithm does not need access to the non-watermarked model or to any model with similar capabilities; the attack can be instantiated with only black-box access to the watermarked model, and white-box use of much weaker open-source models.
To give a ``proof of concept'', we instantiate an implementation of the attack and use it to successfully remove LLM watermarks planted by the schemes of \citet{kirchenbauer2023watermark}, \citet{kuditipudi2023robust}, and \citet{zhao2023provable} while maintaining text quality as judged by GPT4~\citep{openai2023gpt4}.

\begin{figure}[ht]
    \centering
    \includegraphics[width=0.8\textwidth]{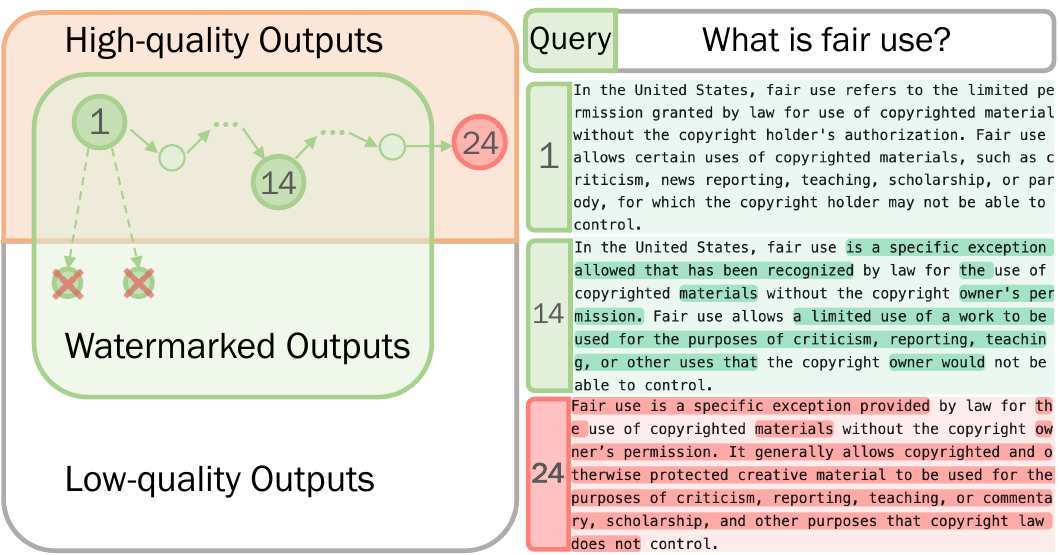}
    
    \caption{\small An outline of our quality-preserving random walk attack schema (The differences with original watermarked text are highlighted.). We consider the set of all possible outputs and within it the set of all high-quality outputs (with respect to the original prompt). For any quality-preserving watermarking scheme with a low false-positive rate, the set of watermarked outputs (\textcolor{caribbeangreen}{green}) will be a small subset of the high-quality output (\textcolor{cadmiumorange}{orange}). 
    We then take a random walk on the set of high-quality outputs to arrive at a non-watermarked output (\textcolor{candypink}{red}) by generating candidate neighbors through the perturbation oracle and using the quality oracle to reject all low-quality candidates. 
    We instantiate our attack for text as follows: given a watermarked text, at each iteration, the malicious user can generate span substitutions using a small masked LM, while making sure the response quality with respect to the user query does not decrease according to a quality oracle such as GPT-3.5 or a reward model.}
    \label{fig:teaser}
\end{figure}

\paragraph{Our assumptions in a nutshell.} Consider a generative model $\model$ that takes as input a prompt $x\in \cX$ to generate an output $y\in \cY$ according to some probability distribution. Suppose that $y$ was watermarked in some way, and we consider a watermark-removing adversary $\adversary$. Our starting point is the following simple but powerful observation: $\adversary$'s goal is \emph{not} to find a non-watermarked $y'$ that is semantically equivalent to $y$; rather, it is sufficient for $\adversary$ to find a non-watermarked $y'$ that has \emph{equivalent quality} to $y$ as a response to the prompt. For example, if $x$ was a prompt to write an essay on some topic, and $y$ is a watermarked essay, then $\adversary$ does not need to find a rephrasing of $y$: it is enough to find another essay that would get the same grade. Given the above, we believe that the watermarking task should be phrased with respect to a prompt-dependent  \emph{quality function} $\quality:\cX\times \cY \rightarrow [0,1]$ that on input a prompt $x$ and response $y$ returns a grade that captures the quality of $y$ \emph{as a response to the prompt} $x$.
Our assumptions are the following (see \Cref{sec:assumptions} for formal statements):

\begin{itemize}
    \item \textbf{Quality oracle:} the attacker has access to a ``quality oracle'' that enables it to efficiently compute $\quality$ on input pairs $(x,y)$ of its choice. The quality oracle only needs to be able to discern quality up to the quality level of the outputs produced by $\model$.

    \item \textbf{Perturbation oracle:} the attacker has access to a randomized  ``perturbation oracle`` $\pO$, such that on an input response $y \in \cY$ and its corresponding prompt $x \in \cX$, $\pO(x,y)$ is a random perturbed response $y'$ (of $x$) which, roughly speaking, approximately satisfies the following conditions: (1) the probability that $\quality(x,y') \geq \quality(x,y)$ is bounded away from zero, (2) the random walk that the oracle $\pO$ induces on the space of high-quality outputs has good mixing properties. (See \Cref{sec:impossibility} and~\Cref{app:extended-impossibility-result} for formal conditions.)  
\end{itemize}

We claim that the two assumptions typically hold in practice. A proof of concept that a sufficiently discerning quality oracle will exist is that the \emph{watermarked model itself} can be used as a quality oracle. We can simply prompt the model to rate the quality of $y$ as a response to $x$. By the heuristic that ``verification is easier than generation'', if a model is powerful enough to generate high-quality outputs, then it should also be powerful enough to check their quality. 
It is likely that \emph{multi-modal generative models}~\citep{alayrac2022flamingo, reed2022generalist, openai2023dalle3} will allow this argument to extend beyond text to other modalities such as images and audio. There are already existing quality metrics in some of these domains~\citep{salimans2016improved, heusel2017gans, hessel2021clipscore, wu2023human, xu2024imagereward}. In general, as models become more capable, the quality oracle assumption can be made stronger.

Regarding the perturbation oracle, note that the perturbation can be quite minor--- for instance, replacing a masked span of a few tokens. The new tokens can be sampled completely at random (in which case the efficiency of the attack suffers with large vocabulary size) or, more efficiently, by resampling the tokens using an open-source (non-watermarked) masked language model, which can be significantly weaker than $\model$. We use the latter approach in our experiments. More sophisticated algorithms are possible: for instance, a second `harmonization' phase could be added in which tokens outside the span are resampled using $\model$. Regardless of the implementation, the underlying principle behind any perturbation oracle is that there is a cloud of high-quality responses that are accessible starting from $y$ through the accumulation of incremental quality-preserving modifications. Watermarking schemes themselves typically rely on the connectivity of large portions of high-quality output space. At one extreme, if the prompt $x$ uniquely defines a single high-quality output (for example, if the prompt asks for the canonically formatted answer to a mathematical problem for which only one solution exists) then the response cannot be watermarked in the first place.

\subsection{Theoretical Result}

Our main theoretical result is the following:

\begin{theorem}[Main result, informal] \label{thm:maininformal} For every (public or secret-key) watermarking setting satisfying the above assumptions, there is an efficient attacker that given a prompt $x$ and (watermarked) output $y$ with probability close to one, uses the quality and perturbation oracles to obtain an output $y'$ such that (1) $y'$ is \emph{not} watermarked with high probability and (2) $\quality(x,y') \geq \quality(x,y)$.
\end{theorem}

Our formal definition of watermarking schemes is given in \Cref{sec:watermarking}, and the formal statement of~\Cref{thm:maininformal} is given in~\Cref{thm:impossibility-simple} (in a simplified form), and fully in~\Cref{thm:impossibility-perturbation} of~\Cref{app:extended-impossibility-result}.
The main idea behind~\Cref{thm:maininformal} is simple: see~\Cref{fig:teaser}.
The adversary uses \emph{rejection sampling} to run a random walk on the set of high-quality outputs of a prompt $x$. 
That is, given the initial response $y'=y$; for $t=0,1,\ldots$, the adversary repeatedly samples $y_t \getsr \perturbationOracle(x,y')$ from the perturbation oracle and accepts $y'$ 
(i.e., sets $y' :=y_t$) if $\quality(x,y') \geq \quality(x,y)$. 
See \Cref{alg:attack} for pseudocode.
The adversary does not need to know any details of the watermarking scheme or the space of outputs beyond an upper bound on the mixing time of the random walk.

We note that our definitions of watermarking schemes assume that there is a statistical signal inserted randomly into the generative model output.
This is different from so-called ``AI detectors'' which can have deterministic detection algorithms.
Injecting randomness is necessary to provide bounds on the false-positive probability of detecting non-model generated outputs.
Indeed, so-called AI detectors suffer from multiple issues.
\citet{wu2023survey} surveyed known systems and concluded that existing detection methodologies do not reflect realistic settings, and their deployment may well cause harm.  
\citet{liang2023gpt} showed that several existing detectors are biased against non-native English writers.
OpenAI's FAQ for educators states that AI detectors do not work and suffer from high rates of false positives \citep{OpenAIFAQ2023}.

Finally, we emphasize that, unlike in typical cryptographic settings, our adversary is \emph{weaker} computationally than the watermarked model it is attacking, and only has access to it as a black box.
This only makes the impossibility result stronger.

\begin{table*}[ht]
\caption{Average results on three watermark schemes before and after our attack, as applied to Llama2-7B model. The GPT-4 judge score is obtained as the average of pairwise comparisons between the perturbed text and the original watermarked output. The score is $+1$ if GPT-4 strongly prefers the perturbed text, $-1$ if it strongly prefers the original, and zero otherwise. The reported value is the average score over hundreds of successfully attacked examples and random ordering of the comparands in the prompt.}
\centering
\vspace{0.1in}
\resizebox{0.73\textwidth}{!}{
\begin{tabular}{cccc}
\toprule \toprule
\multirow{2}{*}{Framework} & \multicolumn{2}{c}{C4 Real News} & \multirow{2}{*}{GPT-4 Judge} \\ \cline{2-3}
  & z-score & p-value & \\ \midrule
KGW \citep{kirchenbauer2023watermark} &  6.236 $\rightarrow$ 1.628 & 0.002 $\rightarrow$ 0.187 & -0.0877 \\
 Unigram  \citep{zhao2023provable} & 8.210 $\rightarrow$ 1.456 &  4.563e-11 $\rightarrow$ 0.208 & -0.0812 \\
 EXP \citep{kuditipudi2023robust} & 3.540 $\rightarrow$ 0.745 &  < 1/5000 $\rightarrow$ 0.3119 & -0.0675 \\
\bottomrule \bottomrule  
\end{tabular}
}
\label{tab:main_tab}
\end{table*}

\begin{figure*}[h]
\begin{minipage}[h]{\textwidth}
\centering
\includegraphics[width=.48\textwidth]{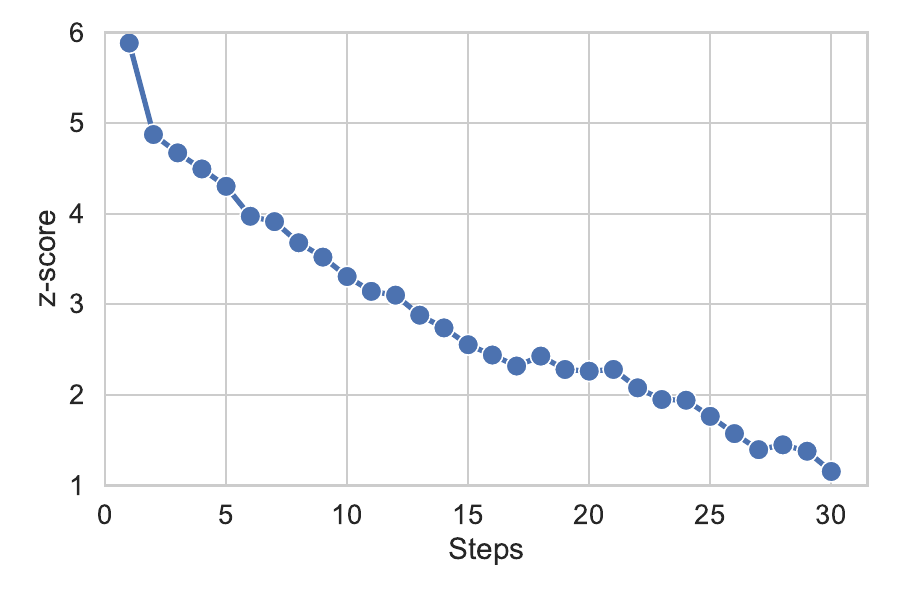} %
\label{fig:z_time}
\includegraphics[width=.48\textwidth]{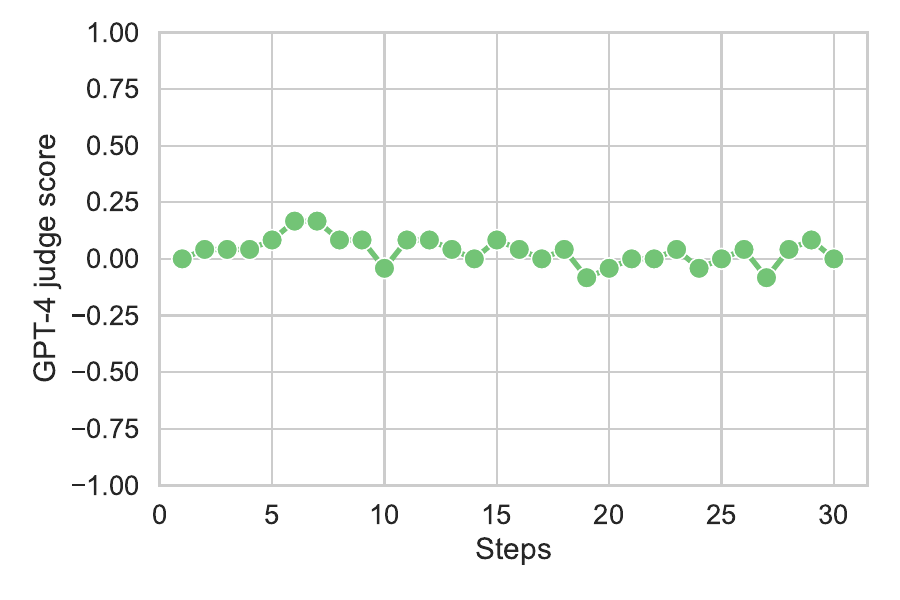}  %
\label{fig:judge_time}
\end{minipage}
\caption{
Detection and quality w.r.t. the number of perturbation steps using Llama2-7B with the KGW scheme \citep{kirchenbauer2023watermark}. Left: z-score (standard deviation deviation from the null hypothesis of non-watermarked content). Right: GPT-4 Judge score. Results are aggregated across 12 examples and the order of comparands.}
\label{fig:time}
\end{figure*}

\subsection{Experimental Results and Implementation of Our Attack}

As proof of concept, we implement instantiations of quality and perturbation oracles for the text modality, resulting in a practical attack against language model watermarking schemes. 
The attack is successful in removing watermarks from the three watermark schemes on which we tested it
\citep{kirchenbauer2023watermark, kuditipudi2023robust,zhao2023provable},\footnote{We selected these schemes because they have publicly available implementations. } see \Cref{tab:main_tab} and \Cref{fig:time}.

Specifically, show that with enough quality-preserving perturbations, we can degrade watermark average detection performance on a C4 \citep{raffel2020exploring} news completion task to a z-score below $1.645$ and a p-value greater than $0.05$, a standard threshold that entails no more than $5\%$ false positive rates.
In addition, although we use reward models and GPT-3.5 as our quality oracles during our attack, to ensure that the attack is not overfitting to these imperfect proxies, we also measure the quality of the perturbed output using GPT-4.
(We stress that GPT-4 is only used to measure the quality of the attack, and not in the attack itself.)
We see that while the detection probability steadily reduces, the quality score is generally stable (Figure~\ref{fig:time}). 
While the attack might not be the most efficient approach for these particular schemes, it has the advantage of being \emph{generic} and not varying 
based on the implementation details of each scheme. 
See \Cref{sec:experiment} for more details on the implementation and experimental results.

\subsection{A Thought Experiment}

We expect that with time, generative AI models would both be able to handle more complex prompts, as well as be able to generate outputs with more modalities.
As intuition for both our attack and our contention that increasing capabilities favor the watermark attacker, consider the following ``mental experiment'', inspired by \citet{valiant1985np}. 
We stress that this mental experiment is \emph{not} how our actual attack is implemented.
Consider a powerful chat-like model that can generate content based on highly complex prompts.
Suppose that the prompt \texttt{"Generate output $y$ satisfying the condition $x$"} has $N$ possible high-quality responses $y_1,\ldots,y_N$ and that some $\epsilon$ fraction of these prompts are watermarked.
Now consider the prompt \texttt{"Generate output $y$ satisfying the condition $x$ and such that $h(y) < 2^{256}/N$"} where $h$ is some hash function (which can be simple and non-cryptographic) mapping the output space into $256$ bits (which we identify with the numbers $\{1,\ldots,2^{256}\}$).
Since the probability that $h(y)<2^{256}/N$ is $1/N$, in expectation there should be a unique $i$ such that $y_i$ satisfies this condition.
Thus, \emph{if} the model is to provide a high-quality response to the prompt, then it would have no freedom to watermark the output.
An attacker could use binary search to find the value of $N$, as well as use random ``salt'' values for the hash function, to ensure a high probability of success.
We note that it is just a mental experiment. 
As of November 2023, current models are unable to evaluate even simple hash functions, and their performance deteriorates with additional conditions.
But we believe it does provide intuition as to how the balance in watermarking shifts from defender to attacker as model capabilities increase.

\begin{figure}[h]
\begin{algorithm}[H]
\caption{Pseudocode for our attack}\label{alg:attack}
\SetKwInput{KwInput}{Input}
\SetKwInput{KwOutput}{Output}
\KwInput{prompt $x$, watermarked response $y$, quality oracle $\quality$, perturbation oracle $\perturbationOracle$, random walk length $T$.}
\KwOutput{response $y'$ without watermark.}

$y' \leftarrow y$ \tcp*{initialize with the watermarked response}

\For{$t \leftarrow 1$ \KwTo $T$}{
    $y_t \leftarrow \perturbationOracle(x, y')$ \tcp*{apply perturbation}
    \If{$\quality(x, y_t) \ge \quality(x, y)$}{
         $y' \leftarrow y_{t}$ \tcp*{update if quality does not decrease} %
    }
}

\Return{$y'$ without watermark} \tcp*{return the de-watermarked response}
\end{algorithm}
\end{figure}

\section{Related Work}\label{sec:related}

\paragraph{Watermarks for generative models.}
Digital watermarking---embedding imperceptible but algorithmically detectable signals in data for the purpose of attributing provenance--has a long history (see \citet{cox2007digital} for a survey) across modalities such as images \citep{o1996watermarking, o1997rotation, cox2007digital, hayes2017generating, zhu2018hidden}, text \citep{atallah2001natural, atallah2002natural}, and audio \citep{boney1996digital,arnold2000audio}. Watermarking of generative model outputs, involving changing the generation process itself instead of simply adding watermarks to generated images post-hoc, is a more recent development which was first introduced in the context of image generators \citep{yu2021artificial} and then in the LLM setting \citep{aaronson2023watermarking, kirchenbauer2023watermark}. Several proposed schemes in the vision setting modify the models themselves through pre-training \citep{yu2021artificial} or fine-tuning \citep{fei2022supervised, zeng2023securing, zhao2023recipe, fernandez2023stable}; the scheme of \citet{wen2023tree} intervenes on a diffusion model's sampling process. Intervening on the sampling process (treating the neural network itself as a black box) is the standard approach for LLM schemes \citep{kirchenbauer2023watermark,kirchenbauer2023reliability,kuditipudi2023robust,christ2023undetectable,liu2023semantic,zhao2023provable}. The scheme in \citet{kirchenbauer2023reliability} boosts the probabilities of tokens from an adaptively chosen ``green'' subset of the vocabulary at each inference step. Follow-up works \citep{kirchenbauer2023reliability,liu2023semantic,zhao2023provable} improve the robustness of this scheme to bounded-edit distance attacks. 
The schemes of \citet{christ2023undetectable} and \citet{kuditipudi2023robust} guarantee the desirable property that the marginal distribution of outputs over the random choice of the secret key is the same as (or computationally indistinguishable from) the un-watermarked distribution, ensuring no quality degradation on average.
\citet{jaiden2023pub} embed a publicly-verifiable cryptographic signature into texts using rejection sampling.

\paragraph{Limitations of generative model watermarks.}
There have also been various recent works that attack watermarking schemes. In the vision domain, there are attacks \citep{zhao2023invisible, lukas2023leveraging,saberi2023robustness} that can erase watermarks for various watermarking schemes; relevant attacks typically involve adding noise to either the images themselves or latent representations, and/or performing some optimization procedure to remove the watermark. Some of these works prove that their attacks will succeed under strict assumptions about watermarks. For instance, the impossibility result of \citet{zhao2023invisible} is focused on classical schemes that apply a bounded-perturbation watermark post-hoc to individual images. Meanwhile, \citet{saberi2023robustness} propose a general attack--- diffusion purification--- for image watermark schemes with a low ``perturbation budget'', which is the stringent requirement that even for a single key, the distribution of watermarked outputs is close to the original distribution of outputs. They also demonstrate a ``model substitution'' attack against more general image watermarking schemes. This attack finds adversarial perturbations with respect to a proxy watermark detector. However, obtaining the proxy detector requires either white-box access to the detection algorithm or a large number of watermarked and non-watermarked samples. 
\citet{sadasivan2023can} and \citet{saberi2023robustness} prove limitations on post-hoc detectors when the distributions of human and AI-generated text are close.
In the LLM domain, \citet{kirchenbauer2023watermark, kirchenbauer2023reliability, kuditipudi2023robust} empirically study the viability of various attacks against the schemes they propose. 
Notably, in \emph{paraphrasing attacks} (which can be implemented using language models, translation systems, or manually by humans) the attacker rewords the watermarked output while preserving its meaning. The above works find that paraphrases often leave some substrings of the text intact, or degrade text quality, and thus have mixed success on schemes that are robust to bounded edit distance attacks. We emphasize that our attack does not preserve the semantics of the output, and is therefore less limited than paraphrasing attacks. A watermarking impossibility result has also been claimed in the LLM domain by \citet{sato2023embarrassingly}; however, their definition assumes security against an unbounded-time adversary that can sample from the distribution of the original unwatermarked generative model (see~\cite[Appendix B]{sato2023embarrassingly}).

\paragraph{Related settings.}
Parallel to the line of work on watermarking generative models, there have been various proposed schemes for post-hoc detection of generative model outputs \citep{zellers2019defending,mitchell2023detectgpt,tian2023gptzero,wang2023bot,verma2023ghostbuster,tian2023gptzero,chakraborty2023possibilities,gehrmann2019gltr,wu2023survey}. When there is no watermark in the data, detecting whether it is generated by a given model is a harder problem, and there are fundamental limitations to these schemes. Post-hoc detection needs to rely on empirical differences between the distributions of generative model outputs and natural/human outputs. As generative models become better at modeling their training distributions, this detection problem can become increasingly difficult; and since adversaries can adapt their outputs to thwart these detectors, there cannot be error rate guarantees in this setting. Empirically, paraphrasing attacks work well against post-hoc detection algorithms for LLMs \citep{krishna2023paraphrasing,kirchenbauer2023reliability}.
Distinct from watermarking schemes that aim to make every output of a generative model detectable, there are also schemes that aim to watermark the weights of the network for the purpose of intellectual property protection, or enable a party to detect a signature in a given network given adaptive black-box access to the network \citep{uchida2017embedding,zhang2018protecting,rouhani2018deepsigns,bansal2022certified}.
Finally, we note that watermarking is a special case of steganography, which is the more general practice of planting hidden information in data. This has also been studied in the context of generative models \citep{fang2017generating, yang2018rnn, ziegler2019neural, tancik2020stegastamp}.

\vspace{-1mm}
\section{Secret-key Watermarking Schemes for Generative Models}\label{sec:watermarking}

In this section, we formalize the notion of a (strong) watermarking scheme for generative models.
First, we begin with the formalization of generative models. Then, we introduce the concept of secret-key watermarking schemes.
We refer the reader to~\Cref{sec:notation} for the notation used.

\subsection{Generative Models}\label{sec:generative-models}
Generative models are randomized algorithms that produce a response $y\in\cY$ (e.g., an image or text) in response to a given input prompt $x\in\cX$ (e.g., a question or a description of an image).
\begin{definition}[Generative models]\label{def:gen-model}
    A \emph{conditional generative model} $\model:\cX \rightarrow \cY$ is a randomized efficient algorithm that, given a prompt $x\in\cX$, produces an output $y \in \cY$.\footnote{A conditional generative model $\model$ is considered efficient if its running time is polynomial in the length of the prompt $x\in \cX$.}
    We call $\cX$ the prompt space, $\cY$ the output space, and $y \getsr \model(x)$ the $\model$'s response to a prompt $x \in \cX$.
\end{definition}

To study watermarking formally, we associate with a generative model $\model$ a \emph{quality function} $\quality$ that assigns a score $0\leq \quality(x,y) \leq 1$ to a response $y$ (generated by $\model$) for a given prompt $x$.
The quality function serves as a measure of the ``quality'' of the generative model's response. Moreover, the quality $\quality$ must be ``universal'' in the sense that it can assign a score even to a response $y$ not generated by the corresponding $\model$.
For instance, $y$ might be a response crafted by a human. In essence, $\quality$ should be objective and not contingent upon the specific generative model under consideration.
\begin{definition}[Quality function]\label{def:quality}
    A {\em quality function} $\quality$ is a deterministic function $\quality:\cX \times\cY \rightarrow [0,1]$ that assigns a score $\quality(x,y) \in [0,1]$ to the response $y\in\cY$ for a given prompt $x \in\cX$.
    We say a quality function $\quality$ is associated with a conditional generative model $\model:\cX \rightarrow \cY$ if the response quality of $\model$ is evaluated using $\quality$. That is, $\quality(x,\model(x))$ is the score assigned to $\model$'s response $\model(x)$ for a prompt $x\in \cX$.
\end{definition}

\subsection{Secret-Key Watermarking Schemes}\label{sec:definitions-watermarking}

In this section we formally define watermarking schemes.
Because our focus is a negative result, we focus on defining \emph{secret-key} watermarking schemes.
These are easier to construct, and hence ruling them out makes our impossibility result stronger.

\begin{definition}[Secret-key watermarking scheme] \label{def:watermarking}
A {\em secret-key watermarking scheme} for a class of generative models $\cM = \classModel$ with a key space $\cK$ consists of the following efficient algorithms:\footnote{The watermarking scheme is considered efficient if both $\watermark$ and $\detect$ run in polynomial-time in the input size.}
\begin{description}
    \item[$\watermark(\model)$:] Given a generative model $\model \in \cM$, this randomized watermarking algorithm outputs a secret key $\k\in\cK$ and a watermarked generative model $\model_\k:\cX \rightarrow \cY$, dependent on $\k$.
    
    \item[$\detect_\k(x,y)$:] Accepting a secret key $\k \in \cK$, a prompt $x \in \cX$, and an output $y \in \cY$, this deterministic detection algorithm returns a decision bit $b \in \bin$, where $1$ indicates the presence of the watermark, and $0$ indicates its absence.
\end{description}
\end{definition}

\noindent \Cref{def:watermarking} has several aspects that make it easier to realize such schemes: $(i)$ we focus on \emph{secret-key} watermarking schemes, where the detection algorithm $\detect$ uses the same secret key $\k$ that embeds the watermark; $(ii)$ we adopt \emph{prompt-conditional detection}, meaning the $\detect$ algorithm also receives the prompt $x$; and $(iii)$ we grant the watermarking algorithm $\watermark$ \emph{non-black-box} access to the model. In other words, the watermarked version $\model_\k$ of $\model$ could be derived by utilizing the internal operations of $\model$.\footnote{For instance, the watermarking scheme might exploit the internal parameters of $\model$ during the watermarking phase.}
These choices simplify the development of the watermarking scheme. Thus, our chosen relaxations only reinforce our impossibility results presented in \Cref{sec:impossibility}.

\paragraph{Properties of watermarking schemes.}
\Cref{def:watermarking} captures only the syntactic conditions for watermarking.
Next, we define the properties of a watermarking scheme, which will be used later to prove our impossibility result.
Primarily, we aim to measure the ``false negative'' and ``false positive'' rates.
The former refers to the proportion of $\model_\k$'s responses that are {\em incorrectly} identified as ``un-watermarked'', i.e., $\detect_{\k}(x,\model_{\k}(x))=0$.
The latter represents the probability of outputs $y \in \cY$ (not produced by $\model_\k$) that are still labeled as ``watermarked''.
The formal definitions are provided below.\footnote{As mentioned in \Cref{sec:intro}, we consider the case of watermarking schemes with controlled false positive probability taken over the choice of the injected randomness (i.e., key) and do not consider deterministic ``AI detectors''.}

\begin{definition}[False negative and false positive $\epsilon$-rates]\label{def:false-positive-negative}
Let $\epos, \eneg > 0$ and $\Pi=(\watermark, \detect)$ be a secret-key watermarking scheme for a class of generative models $\cM = \classModel$.
\begin{description}
     \item[False negative $\eneg$-rate:] $\Pi$ has a \emph{false negative $\eneg$-rate} if, for every model $\model\in \classModel$ and every prompt $x \in \cX$, $\Pr[ \detect_{\k}(x,\model_{\k}(x))=0 ] \leq \eneg$. The probability is over the random coins of $\model_\k$ and the pair $(\k,\model_\k)$ output by $\watermark(\model)$.
    \item[False positive $\epos$-rate:] $\Pi$ has a \emph{false positive $\epos$-rate} if, for every model $\model\in \classModel$, for every prompt $x \in \cX$, and for every output $y \in \cY$, $\Pr[ \detect_{\k}(x,y)=1 ] \leq \epos$. The probability is over the pair $(\k,\model_\k)$ output by $\watermark(\model)$.
\end{description}
The rates above might be influenced by $x\in\cX$ and $y \in \cY$, for instance, they may diminish with the length of the output.    
\end{definition}

\noindent Broadly, a low false negative $\eneg$-rate indicates that the watermarking scheme is functioning as anticipated. This means there is a significant probability that the detection algorithm outputs $1$ when $y$ is derived from the watermarked model.
Conversely, small false positive $\epos$-rate implies there is a small likelihood for $\detect_\k(x,y)=1$ when $\k$ is sampled independently of $(x,y) \in \cX$.
Introducing a random key $\k$ ensures that we can bound the false positive rate without needing to make assumptions on the unknown human-generated data distribution.
Practical watermarking schemes typically have a tunable hyperparameter (such as a $z$-score threshold or $p$-value) that enables trading in the false positive and false negative rate.
We remark that for many applications, watermarking schemes require a very low false positive rate. For example, if professors routinely use watermark detectors for testing for cheating in hundreds or thousands of problem-set or essay submissions, then unless $\epos$ will need to be very small, there would be a high chance of students being falsely blamed.

\paragraph{Security of watermarking schemes.}
\Cref{def:false-positive-negative} is insufficient for ensuring security as it only characterizes the basic properties that a watermarking scheme should possess.
For instance, these definitions might be met by watermarking an output by appending the sentence ``This text was generated by an AI model.'' to the end of the LLM's response.\footnote{Here, we are implicitly assuming that the appended sentence does not significantly degrade the quality of the response.}
Regrettably, such an approach lacks robustness since anyone can effortlessly remove the watermark by deleting the appended sentence.
This consideration leads us to the need to define security. Specifically, the watermark should be \emph{robust} against adversarial attempts at removal.
The adversary's objective is to strip the watermark from $y$ (i.e., evade detection) while maintaining the quality (as assessed by the quality function $\quality$) near the original response $y$ level. 
This requirement is crucial; otherwise, any arbitrary response could constitute a valid erasure attack.
Given our aim is to establish an impossibility result, we will set forth a basic condition that is necessary but not sufficient. 
The formal definition follows.

\begin{definition}[Erasure attack against watermarking schemes]\label{def:erasure-attack}
    Let $\Pi = (\watermark,\detect)$ be a watermarking scheme for a class of generative models $\cM = \classModel$ with associated quality function $\quality:\cX\times\cY \rightarrow [0,1]$.
    We say that an adversary $\adversary$ \emph{$\epsilon$-breaks} $\Pi$ if for every $\model \in \cM$, for every prompt $x\in\cX$, we have
    \[
        \prob{
            \begin{array}{c}
                \detect_\k(x,y') = 0 \text{ and } \quality(x,y') \geq \quality(x,y)
            \end{array}
            :
            \begin{array}{c}
                y \getsr \model_\k(x),\ y' \getsr \adversary(x,y)
            \end{array}
        } \geq \epsilon
    \]
    where the probability is taken over $(\k,\model_\k)$ output by $\watermark(\model)$ and the random coins of $\adversary$.
\end{definition}

\noindent Note that in a secret-key watermarking scheme (where an attacker does not know $\k$), the highest probability we can anticipate for success in an erasure attack is $1-\epos$, where $\epos$ is the false positive rate (see~\Cref{def:false-positive-negative}).

Intuitively, a watermarking scheme $\Pi = (\watermark,\detect)$ is considered robust and secure (for a meaningful choice of $\epsilon$) if there does not exist a computationally efficient adversary $\adversary$ that satisfies the aforementioned~\Cref{def:erasure-attack}.
The notion of ``computationally efficient'' depends on the context, but at the very least the adversary should be able to: (1) make black-box queries to the watermarked model itself, (2) perform efficient computations that do not require white-box access to the model or the data that it was trained on.
As we demonstrate in~\Cref{sec:impossibility}, there is a practical and universal adversary that meets~\Cref{def:erasure-attack}'s criteria in several settings.
However, if the adversary is restricted in other ways (for example only allowed to modify the output up to $\epsilon$ distance in some metric) then it may be possible to resist erasure attacks.
This is the setting of ``weak watermarking'' that we mention in \Cref{sec:intro}.
We do not investigate weak watermarking schemes in this paper.

\paragraph{On quality approximation/degradation.}
According to~\Cref{def:erasure-attack}, an attack is considered successful when the quality of the non-watermarked output $y'$ (computed by the adversary) is at least that of the target watermarked output $y \getsr \model_\k(x)$, i.e., $\quality(x,y') \geq \quality(x,y)$.
We highlight that we can relax~\Cref{def:erasure-attack} to consider adversaries that can remove the watermark at the price of some ``small'' degradation in the quality.
This can be accomplished by modifying~\Cref{def:erasure-attack} and requiring that $\quality(x,y')\geq \quality(x,y) - \gamma$ (instead of $\quality(x,y')\geq \quality(x,y)$) where $\gamma$ is the parameter defining the quality degradation tolerated by the definition.
For simplicity, in this work we focus on the simpler setting in which there is no degradation, and the attacker has access to an exact quality oracle.
In practice, the quality oracle may not be perfect, and hence some approximation and degradation might be necessary.

\vspace{-1mm}
\section{Impossibility Result}\label{sec:impossibility}

We now show that a secret-key watermarking scheme can be defeated (i.e., remove the watermark) when the outputs of the underlying generative model can be perturbed.
With the term ``perturbed'' we mean that the output can be slightly modified without affecting the quality.
Examples of perturbable outputs are texts in which a word/token/span can be modified without changing the meaning of the text, or images in which a single pixel/patch can be modified (e.g., slightly change the RGB value of the target region) without affecting the picture. Thus, the impossibility result of this theorem applies to LLMs or image models.

\subsection{Assumptions and Discussion} \label{sec:assumptions}

Our impossibility result relies on the following two  assumptions:
\begin{enumerate}
\item \textbf{Adversaries can check their own work:} The adversary has access to the quality oracle $\quality:\cX\times\cY \rightarrow [0,1]$ (associated to the target generative model $\model: \cX\rightarrow\cY$).

\item \textbf{Random walk on the output space $\cY$:} At a high level, we assume that there is some graph $\graph = (\cV,\cE)$ in the output space $\cV = \cY$ of the target watermarked model $\model_\k:\cX\rightarrow\cY$. In addition, the adversary has access to an oracle $\pO$ (dubbed {\em perturbation oracle}) that, on input $y\in \cY$ (e.g., the watermarked output), returns a random neighbor $y'$ of $y$ according to the graph $\graph$. For example, in the case of language, this can be obtained by taking some subset of the tokens and replacing them with random choices (either uniform or informed by large or smaller language models) or semantically equivalent spans.
The main assumption we make is that the random walk has non-trivial \emph{mixing} properties, and in particular that the graph is \emph{irreducible} and \emph{aperiodic} (these two conditions guarantee that a long-enough random walk will converge to its stationary distribution).

\end{enumerate}

\noindent The definition of a perturbation oracle is given below.

\paragraph{Pertubation oracle.}
We abstract the ability of the adversary to perform a random walk over the output space of a (possibly watermarked) generative model $\model:\cX \rightarrow \cY$ by introducing the notion of {\em perturbation oracles}.
In a nutshell, a perturbation oracle $\perturbationOracle:\cX \times \cY \rightarrow \cY$ is a randomized oracle that, on input $x \in \cX$ and $y \in \cY$, returns a high-quality $y'\in\cY$ (according to some distribution) with probability at least $\epsilon_{\sf pert}$.
The formal definition follows.
\begin{definition}[Perturbation oracle]\label{def:perturbation-oracle}
    A perturbation oracle $\perturbationOracle$ is a randomized oracle that, on input $x \in \cX$ and $y\in\cY$, returns an $y' \in \cY$.
    We say a perturbation oracle $\perturbationOracle:\cX \times \cY \rightarrow \cY$ with quality function $\quality:\cX\times\cY \rightarrow [0,1]$ is {\em $\epsilon_{\sf pert}$-preserving} if for every prompt $x \in \cX$, for every $y \in \cY$, we have 
        \[
            \prob{\quality(x,\perturbationOracle(x,y)) \geq \quality(x,y)} \geq \epsilon_{\sf pert}.
        \]
\end{definition} %

\subsubsection{Mixing Condition} 

To define the mixing condition, we need to define an associated random walk for the perturbation oracle with respect to some particular prompt.
We then consider the \emph{mixing time} of this random walk   (see~\Cref{sec:preliminaries} for a formal definition of mixing time).
Fix a prompt $x\in\cX$.
The perturbation oracle $\perturbationOracle(x,\cdot)$ can be represented using a weighted directed graph $\graph_x=(\cV_x,\cE_x)$ where an edge $(y_0,y_1)\in\cE_x$ if weighted with the probability of going from output $y_0$ to output $y_1$ using $\perturbationOracle(x,\cdot)$.
Below, we formally define two different graph (hierarchically ordered) representations corresponding to $\perturbationOracle$.
The first, named  the{\em $x$-prompt graph representation}, is essentially the graph $\graph_x=(\cV_x,\cE_x)$ described above, i.e., the graph representing the perturbation oracle $\perturbationOracle(x,\cdot)$ for a fixed prompt $x\in\cX$.
The second, named {\em $q$-quality $x$-prompt graph representation}, is the subgraph $\graph^{\geq q}_x = (\cV^{\geq q}_x,\cE^{\geq q}_x)$ of $\graph_x=(\cV_x,\cE_x)$ where we only consider vertices of quality at least $q$, i.e., $\quality(x,y) \geq q$.

\begin{definition}[Graph representation of perturbation oracles]\label{def:graph-perturbation-oracle}
    Let $\perturbationOracle:\cX \times \cY \rightarrow \cY$ be a perturbation oracle with associated quality function $\quality: \cX \times \cY \rightarrow [0,1]$.
    \begin{description}

        \item[$x$-prompt graph representation:] Fix a prompt $x \in \cX$, the {\em $x$-prompt} graph representation of $\perturbationOracle$ is a weighted directed graph $\graph_x = (\cV_x,\cE_x)$ such that $\cV_x = \cY$ and $\cE_x \subseteq \cY \times \cY$ defined as follows:
        \begin{align*}
            \cE_x &= \left\{ (y_0,y_1) \in \cY \times \cY \ : \ \prob{ y_1 = \perturbationOracle(x,y_0)} > 0 \right\}. 
        \end{align*}
        The weight function $\weight(y_0,y_1)$ of $\graph_x = (\cV_x,\cE_x)$ is defined as follows: 
        \[
            \weight(y_0,y_1) = \left\{
                \begin{array}{ll}
                    \prob{ y_1 = \perturbationOracle(x,y_0)} & \text{ if } (y_0,y_1) \in  \cE_x, \\
                    0 & \text{otherwise.}
                \end{array}
            \right.
        \]

        \item[$q$-quality $x$-prompt graph representation:] Fix $q \in [0,1]$ and a prompt $x \in \cX$, the {\em $q$-quality $x$-prompt} graph representation of $\perturbationOracle$ is a graph $\graph^{\geq q}_x = (\cV^{\geq q}_x,\cE^{\geq q}_x)$ such that $\cV^{\geq q}_x \subseteq \cY$ and $\cE^{\geq q}_x \subseteq \cY \times \cY$ defined as follows:
        \begin{align*}
            \cE^{\geq q}_x &= \left\{ (y_0,y_1) \in \cY \times \cY \ : \ \quality(x,y_0) \geq q,\ \quality(x,y_1) \geq q, \text{and } \prob{ y_1 = \perturbationOracle(x,y_0)} > 0 \right\}, \text{ and} \\ 
            \cV^{\geq q}_x &= \left\{ y \in \cY \ : \quality(x,y) \geq q\right\}.
        \end{align*}   
        The weight function $\weight(y_0,y_1)$ of $\graph^{\geq q}_x = (\cV^{\geq q}_x,\cE^{\geq q}_x)$ is defined as follows: 
        \[
            \weight(y_0,y_1) = \left\{
                \begin{array}{ll}
                    \prob{ y_1 = \perturbationOracle(x,y_0)} & \text{ if } (y_0,y_1) \in \cE^{\geq q }_x, \\
                    0 & \text{otherwise.}
                \end{array}
            \right.
        \]
    \end{description}
\end{definition}

\subsubsection{Discussion}
The quality-oracle assumption is based on the intuition that it is easier, or at the very least not harder, to verify outputs than to generate them.
In practice, we have used both zero-shot prompting (see \Cref{sec:experiment}) and reward models to compute the quality function.
In the context of images, it is also possible to use measures such as Inception score \citep{salimans2016improved}, FID score \citep{heusel2017gans}, CLIP score \citep{hessel2021clipscore} or text-to-image reward models \citep{wu2023human, xu2024imagereward} etc.
While at the moment zero-shot prompting works only for language models and not for audio/visual data, as models grow in power, both in capabilities and modalities, computing quality functions will only become easier.

The perturbation-oracle assumption is more subtle. First, note that the notion of ``perturbation'' can be abstract and does not need to correspond to changing a discrete subset of the output (e.g., a sequence of tokens or a patch of an image).
The assumption can potentially fail for some prompts. For example, if a prompt has the form ``Solve for the following set of $n$ linear equations in $n$ variables, and format the solution as ...'' then there is zero entropy in the space of possible high-quality solutions.
In such a case there is no ``perturbation oracle'' that can be defined. 
However, in such a case no watermarking is possible either.
Indeed, watermarking schemes require some sort of ``perturbation oracle'' as well. 
To ensure successful watermarking, there needs to be a large set of potential high-quality solutions, and the scheme needs to be able to select a small subset of it (with a measure bounded by $\epos$--- the false positive rate) based on the key.
Hence we believe this assumption will be justified in cases where watermarking is feasible as well.
Note that our perturbation oracle is in one crucial case much weaker than what is needed for watermarking: we do \emph{not} require it to always preserve quality or even do so with high probability. It is enough that the probability of preserving quality is bounded away from zero.

Another potential issue with the perturbation oracle is that there could be outputs that are high-quality but would be very unlikely to be ever output from the model, whether watermarked or not. Hence the graph might be disconnected simply because of these ``unreachable outputs''. However, these can be addressed by introducing weights on the vertices as well that correspond to the probability of the output under the model. 
Such weighing makes no difference to our argument, but we omit this for clarity of notation.

\subsection{Proof Overview of the Impossibility Result}

The idea behind the impossibility result is to allow the adversary to perform a random walk on the graph while walking only over vertices with sufficiently high quality. Such a high-quality random walk can be executed by leveraging the quality oracle $\quality$ and checking the quality of the next perturbed vertex (at each step of the random walk).
Somewhat more formally, let $x$ be a prompt and $\cV^{\geq q}_x$ be a subset of vertices (of the graph) of quality $\quality(x,y)\geq q$.
Assuming the vertices $\cV^{\geq q}_x$ are \emph{connected} (i.e., there is a non-zero probability of reaching any vertex in $\cV^{\geq q}_x$), then an adversary, starting from a watermarked output $y \in \cY$ with quality $\quality(x,y) \geq q$, can perform a random walk (using its corresponding oracle) to reach a $y' \in \cV^{\geq q}_x$ that is not watermarked.
As we will see next, the existence of such a non-watermarked $y' \in \cV^{\geq q}_x$ is guaranteed by the false positive rate of the watermarking scheme (see~\Cref{def:false-positive-negative}).

\paragraph{Impossibility result.}
For simplicity, let us assume there exists a perturbation oracle $\perturbationOracle$ that is $1$-preserving, i.e., for a fixed prompt $x \in \cX$, the oracle \emph{always} outputs a perturbed output $y' \in \cY$ whose quality is at least that of the initial value $y \in \cY$.
Now, consider an adversary $\adversary$ with oracle access to $\perturbationOracle$.
The adversary's objective is to perturb the watermarked output $y_0$ $t$ times (for some large enough $t$) to produce intermediate perturbations $y_1, \ldots, y_{t}$, such that the final perturbed output $y_t$ will be independent of the watermarking secret key $\k$. In this way, $y_t$ will likely be verified as non-watermarked (i.e., $\detect_\k(x,y_t) = 0$) with probability at least $1-\epos$ where $\epos$ is the false positive rate of the watermarking scheme.
This can be achieved by requiring the adversary to compute $y_i \getsr \perturbationOracle(x, y_{i-1})$ for $i \in [t]$, where $y_0$ corresponds to the initial watermarked output.
This adversarial strategy is effective in removing the watermark for the following two reasons:
\begin{itemize}
    \item The interaction with the perturbation oracle can be represented as a $t$-step random walk over the $q_0$-quality $x$-prompt graph representation $\graph^{\geq q_0}_x = (\cV^{\geq q_0}_x, \cE^{\geq q_0}_x)$ of $\perturbationOracle$ where $q_0$ is the quality of the starting (watermarked) output $y_0$. Thus, $y_0, \ldots, y_t$ exactly represents the path taken by the adversary during such a random walk (note that each $y_1, \ldots, y_t$ will have the same quality $q_0$ by definition of a $1$-preserving perturbation oracle. See~\Cref{def:perturbation-oracle}).
    
    \item If $\graph^{\geq q_0}_x$ is irreducible and aperiodic, a random walk over $\graph^{\geq q_0}_x$ converges to its unique stationary distribution $\vector{\pi}$ (see~\Cref{sec:graphs}).\footnote{Irreducibility and aperiodicity are the fundamental mixing properties that $\perturbationOracle$ (and its corresponding graph) must satisfy.} In addition, if $t$ corresponds to (or is slightly larger than) the $\edist$-mixing time of $\graph^{\geq q_0}_x$, then the distribution of the final perturbed output $y_t$ is $\edist$-close to the stationary distribution $\vector{\pi}$ of $\graph^{\geq q_0}_x$. Furthermore, this distribution $\vector{\pi}$ is completely independent of the secret key $\k$ of the watermarking scheme since the secret key $\k$ is sampled (by the watermarking algorithm $\watermark(\model)$) independently from $\perturbationOracle$'s definition.
\end{itemize}
The above two facts, combined with the false positive $\epos$-rate of the watermarking scheme, directly imply that $y_t$ will be classified as non-watermarked with probability at least $(1-\epos)(1-\edist)$.
Moreover, we can set $\edist$ to be arbitrarily small to achieve an adversarial advantage of $(1-\epos)(1-\edist) \approx (1-\epos)$.\footnote{Indeed, $\edist$ approaches $0$ exponentially as $t$ increases (see~\Cref{thm:mixing-time}).}

The aforementioned approach relies on the fact that the perturbation oracle is $1$-preserving; thus, each oracle invocation will certainly correspond to a new step of the random walk over the graph $\graph^{\geq q_0}_x = (\cV^{\geq q_0}_x,\cE^{\geq q_0}_x)$. This is because, starting from the watermarked output $y_0$ with quality $q_0$, each invocation will output (with probability $1$) a perturbed output $y_i$ with at least the same quality $q_0$.
To make the impossibility more generic, we extend the attack to deal with perturbation oracles that may return low-quality perturbed values with {\em non-zero} probability, i.e., the perturbation oracle is $\epsilon_{\sf pert}$-preserving for some $\epsilon_{\sf pert} < 1$.
In such a setting, we are not guaranteed that each invocation of $\perturbationOracle$ will correspond to a step of the random walk over $\graph^{\geq q_0}_x$. Indeed, with probability $1-\epsilon_{\sf pert}$, the $\epsilon_{\sf pert}$-preserving perturbation oracle may return $y' \in \cY$ such that $\quality(x,y') < q_0$, i.e., $y' \not\in \cV^{\geq q_0}_x$.
To overcome this difficulty, we give the adversary oracle access to the quality function $\quality$.
In such a way, at each iteration $i \in [t]$, the adversary can check the quality of the perturbed value $y_i \getsr \perturbationOracle(x,y_{i-1})$ and proceed as follows:
\begin{itemize}
    \item If $\quality(x,y_i) \geq q_0$, the adversary knows that the obtained perturbed value $y_i$ corresponds to a new random walk step over $\graph^{\geq q_0}_x$. Hence, in the next iteration, it will perturb the recently obtained output $y_i$ (as described previously for the case of a $1$-preserving perturbation oracle).
    \item On the other hand, if $\quality(x,y_i) < q_0$, the value $y_i$ {\em does not correspond} to a new random walk step. In this case, the adversary needs to re-send $y_{i-1}$ to $\perturbationOracle$ until it obtains a perturbed value $y_i$ such that $\quality(x,y_{i}) \geq q_0$.
\end{itemize}
Hence, by accessing the quality oracle $\quality$, the adversary can continue walking over the high-quality graph $\graph^{\geq q_0}_x$ even if the perturbation oracle is not errorless (i.e., $\epsilon_{\sf pert} < 1$).
Naturally, when $\epsilon_{\sf pert} < 1$, the adversary needs to submit more queries to the perturbation oracle to overcome the low-quality outputs of $\perturbationOracle$.
This is taken into account by our theorem which relates the number of queries needed to the parameter $\epsilon_{\sf pert}$, i.e., the higher (resp. lower) $\epsilon_{\sf pert}$, the lower (resp. higher) the number of queries.

\Cref{thm:impossibility-simple} reports our impossibility result that bounds $\adversary$'s advantage $\epsilon$ in erasing the watermark (for every $\model \in \cM$ in a particular class $\cM = \classModel$) when $\adversary$ has oracle access to a perturbation oracle $\perturbationOracle$.
As discussed in this section, such a result relies on the fact that $\perturbationOracle$'s graph representation is irreducible and aperiodic, which are necessary conditions to converge to its unique stationary distribution during the random walk.
To make explicit the independence between the adversary (including its perturbation oracle $\perturbationOracle$) and the secret key $\k$ of the target watermarked model $\model_\k$, we define $\perturbationOracle$'s irreducibility and aperiodicity properties w.r.t. the median quality among those of all possible responses (to a prompt $x$) that can be obtained by watermarking $\model \in \cM$.
Observe that the median quality does not depend on the secret key $\k$ used during the watermarking process but instead depends on \emph{all possible choices} of secret keys (i.e., the key space of $\watermark$).
Formally, let $\cQ_{\model,x} = \{q_1,q_2,\ldots\}$ and $q_{\sf min}$ be defined as follows:
\begin{align}
    \cQ_{\model,x} &= \left\{q : \prob{\quality(x,\model_\k(x)) = q: (\k,\model_\k)\getsr \watermark(\model)}>0\right\},\label{eq:quality-set} \\
    q_{\sf min} &= \underset{\model\in\cM,x\in\cX}{\min}\{q_{\model,x}\} \text{ where } q_{\model,x} \text{ is the median quality of } \cQ_{\model,x},\label{eq:minimum-median}
\end{align}
where the probability in~\Cref{eq:quality-set} is taken over the random coins of $\watermark$ and $\model_\k$.\footnote{This is equivalent to saying that the quality values in $\cQ_{\model,x} = \{q_1,q_2,\ldots\}$ are defined over all possible choices of secret keys $\k$ (output by $\watermark$) and all possible random coins of the resulting watermarked model $\model_\k$.} 
Then, the following~\Cref{thm:impossibility-simple} defines $\perturbationOracle$'s properties w.r.t. $q_{\sf min}$, i.e., the minimum among all median qualities (see~\Cref{eq:minimum-median}).
\begin{theorem}\label{thm:impossibility-simple}
    Let $\Pi=(\watermark,\detect)$ be a watermarking scheme for a class of generative models $\cM = \classModel$ with an associated quality function $\quality:\cX\times\cY \rightarrow [0,1]$.
    Let $\perturbationOracle:\cX\times\cY \rightarrow \cY$ be a perturbation oracle (defined over the same prompt space $\cX$ and output space $\cY$ as the class $\cM$) with the same associated quality function $\quality:\cX\times\cY \rightarrow [0,1]$ as $\Pi$.

Under the following conditions
\begin{enumerate}
    \item The watermarking scheme $\Pi$ has a false positive $\epsilon_{\sf pos}$-rate;
    
    \item The perturbation oracle $\perturbationOracle$ is $\epsilon_{\sf pert}$-preserving;

    \item For every non-watermarked model $\model \in \cM$, for every prompt $x\in\cX$, for every quality $q\in[q_{\sf min},1]$, the $q$-quality $x$-prompt graph representation $\graph^{\geq q}_x$ of $\perturbationOracle$ is irreducible and aperiodic where $q_{\sf min}$ is the minimum median defined in~\Cref{eq:minimum-median}.
    Also, let $t_{\model,x}$ be the $\edist$-mixing time of $\graph^{\geq q}_x$ where $\edist \approx 0$, and let $t_{\sf max}$ be the largest mixing time among $\{t_{\model,x}\}_{\model \in \cM, x \in \cX}$;
    \label{itm:simple-graph}
\end{enumerate}
there exists an \emph{oracle-aided} universal adversary $\adversary^{\perturbationOracle(\cdot,\cdot),\quality(\cdot,\cdot)}$ that $\epsilon$-breaks $\Pi$ (\Cref{def:erasure-attack}) by submitting at most $O(t_{\sf max}/\epsilon_{\sf pert})$ queries to $\perturbationOracle$ where $\epsilon \approx (1-\epos)/2$.
\end{theorem}

\noindent The following corollary is obtained from~\Cref{thm:impossibility-simple} assuming $(i)$ the perturbation oracle is $\frac{1}{2}$-preserving, and $(ii)$ the watermarking scheme has a false positive rate of $\frac{1}{10}$.

\begin{corollary}
Let $\Pi = (\watermark,\detect)$ be a watermarking scheme and $\perturbationOracle$ be a perturbation oracle as defined in~\Cref{thm:impossibility-simple}.
If $\Pi$ has a false positive $\frac{1}{10}$-rate and $\perturbationOracle$ is $\frac{1}{2}$-preserving, then there exists an oracle-aided universal adversary $\adversary^{\perturbationOracle(\cdot,\cdot),\quality(\cdot,\cdot)}$ that breaks $\Pi$ with a success probability of approximately $\frac{9}{20}$, by submitting at most $O(t_{\sf max})$ queries to $\perturbationOracle$, where $t_{\sf max}$ is defined in~\Cref{thm:impossibility-simple}.
\end{corollary}

\paragraph{On improving $\adversary$'s advantage.}
In~\Cref{thm:impossibility-simple}, $\adversary$'s advantage is  approximately $\frac{1-\epos}{2}$ whereas $1-\epos$ is the maximum advantage that an adversary can achieve. The multiplicative $\frac{1}{2}$ loss is due to the definition of $\perturbationOracle$ w.r.t. to the minimum median $q_{\sf min}$. 
Without loss of generality, it is possible to get an adversarial advantage $\epsilon$ arbitrarily close to $1-\epos$ by enforcing $\perturbationOracle$'s properties over a larger range of $q$ values (\Cref{itm:simple-graph} of~\Cref{thm:impossibility-simple}).
This can be accomplished by decreasing $q_{\sf min}$ (instead of using the median quality) such that convergence will be guaranteed for {\em almost} all watermarked outputs $y$ (the ones with quality at least $q_{\sf min}$).\footnote{For example, by setting $q_{\sf min} = 0$, we obtain that the attack succeeds independently from the quality of the received initial watermarked output $y \getsr \model_\k(x)$. This is because the quality of $y$ is always non-negative and, for every quality $q \in [0,1]$, the $q$-quality $x$-prompt graph representation of $\perturbationOracle$ will be irreducible and aperiodic.}
For this reason, in~\Cref{app:extended-impossibility-result}, we include~\Cref{thm:impossibility-perturbation} (and its corresponding proof) that is the extended version of the above theorem.
In particular,~\Cref{thm:impossibility-perturbation} defines $\perturbationOracle$'s properties w.r.t. $v$-th quality percentile (instead of the median) so that, by decreasing $v$, we can get arbitrarily close to the best possible adversarial advantage $1-\epos$.
Moreover,~\Cref{thm:impossibility-perturbation} explicates the concrete relation between the adversarial advantage $\epsilon$, the number of queries, and the $\epsilon_{\sf pert}$-preservation of $\perturbationOracle$. We refer the reader to~\Cref{app:extended-impossibility-result} for more details.

\vspace{-1mm}
\section{Experiments}
\label{sec:experiment}
\vspace{-1mm}
\subsection{Attack Implementation}

As a proof of concept empirical test of our attack schema, we construct an implementation of it and use it to attack several published watermarking schemes.
While our general attack framework applies across all modalities, our focus is language, where quality oracles are currently more straightforward to implement.
To implement our attack, we need to choose (1) the generative model and (2) the watermarking scheme, as well as implement (3) the perturbation oracle, and (4) the quality oracle.

We choose Llama2-7B \citep{touvron2023llama} as our generative model and attack three different watermarking schemes \citep{kirchenbauer2023watermark,kuditipudi2023robust,zhao2023provable}.
We implement our \emph{perturbation oracle} using T5-XL v1.1 \citep{raffel2020exploring}.\footnote{This model is only pre-trained on C4 excluding any supervised training, to mask and infill one random span at each iteration. This design choice might be important as we want to generate a non-watermarked example without substantially shrinking the text length while the original T5 tends to infill texts with shorter lengths than the masked contents. } This model occasionally misses simple errors in the text such as repetitiveness, capitalization and punctuation mistakes, and incoherent or irrelevant phrases, so we use a small Gramformer to do post-hoc grammatical error correction or ask GPT-3.5 to perform a final check of the response quality with a focus on these sorts of errors, whenever the output is approved by the reward model (see~\Cref{fig:error_prompt} for the prompt).\footnote{\url{https://github.com/PrithivirajDamodaran/Gramformer}}
After trying out various potential reward model implementations, we settled on using an open-source reward model (RoBERTa-v3 large \citep{liu2019roberta} fine-tuned on OpenAssistant \citep{kopf2023openassistant} preference data) as the primary quality oracle.\footnote{\url{https://huggingface.co/OpenAssistant/reward-model-deberta-v3-large-v2}}. For removing image watermarks, we use the latest stable-diffusion-xl-base and its distilled version sdxl-turbo. We implement perturbation and quality oracles as stable-diffusion-2-base \citep{Rombach_2022_CVPR} and the reward model trained on Human Preference Score v2 \citep{wu2023human},\footnote{\url{https://huggingface.co/adams-story/HPSv2-hf}} respectively. Experiments were all performed on a combination of 40 GiB and 80 GiB A100s and our code is available at \url{https://github.com/hlzhang109/impossibility-watermark}.

This model occasionally misses simple errors in the text such as capitalization and punctuation mistakes and incoherent or irrelevant phrases, so we ask GPT-3.5 to perform a final check of the response quality with a focus on these sorts of errors, whenever the output is approved by the reward model (see~\Cref{fig:error_prompt} for the prompt).\footnote{For each judgment, we calculate the scores of the original watermarked response and candidate response by passing the output of the reward model through a softmax layer, since this is how the reward model was supervised. We count the comparison as a tie if the difference in scores is less than $\Delta=0.02$.

}
See~\Cref{app:attack_details} for more details on the experimental setup.

All three of the watermarking schemes we attack were originally tested on the Real News subset of the C4 dataset \citep{raffel2020exploring}, so we use this as our primary task as well. Specifically, we give Llama-2-7B \citep{touvron2023llama} the task of generating a completion given the first 20 tokens of a news article. Except where otherwise noted, we default to a generation length of 200 tokens.

We test the following three watermark frameworks and refer readers to~\Cref{app:baselines} for more details: 
\begin{itemize}
    \item KGW~\citep{kirchenbauer2023watermark} selects a randomized set of ``green'' tokens before a word is generated, and then softly promotes the use of green tokens during sampling, which can be detected efficiently. 
    \item  EXP~\citep{kuditipudi2023robust} is a distortion-free watermark framework that preserves the original LM’s text distribution (over the randomness of the secret key), at least up to some maximum number of generated tokens.
    \item Unigram \citep{zhao2023provable} is a watermarking scheme based on KGW designed to be provably robust to perturbations with bounded edit distance from the original watermarked text.
    \item Stable Signature \citep{fernandez2023stable} refers to a method of embedding invisible watermarks into images generated by Latent Diffusion Models (LDMs). This approach involves fine-tuning the latent decoder part of the image generator, conditioning it on a binary signature. 
    \item Invisible watermark \citep{shield2023ivw} is a default (classic) watermark to the Stable Diffusion model series, which utilizes frequency space transformations to embed watermarks invisibly into images, using Discrete Wavelet Transform and Discrete Cosine Transform.
\end{itemize}

\subsection{Experimental Results}
We include our results on C4 news completion tasks with three watermark frameworks in~\Cref{tab:main_tab}. 
We show that with enough iterations, we can degrade watermark average detection performance to a z-score below $1.645$ and a p-value greater than $0.05$, a standard threshold that entails no more than $5\%$ false positive rates. We also query GPT-4 to evaluate the comparative quality of the original watermarked output and the output after our attack. To do this, we prompt GPT-4 to compare the two outputs on a scale from 1 to 5 (corresponding to whether GPT4 ``strongly'' prefers one output, ``slightly`` prefers it, or judges them to ``have similar quality'', see  \Cref{fig:oracle_5choice_prompt}). Because GPT-4 has a bias towards preferring the first of the two outputs \citep{zheng2023judging}, we query it twice, with both orderings of the two outputs.

We take a closer look at the KGW results by plotting histograms of the detection and GPT-4 judgment statistics for around 200 examples (\Cref{fig:hist}). These demonstrate that our attack can substantially degrade detection performance while typically inducing only minor quality degradation. Note that we choose the number of random walk steps for our attack to be high enough to push the average p-value above $0.05$; we expect that running for longer would result in even higher p-values.

\begin{figure}[htbp]
  \centering
  \begin{subfigure}[b]{0.66\textwidth} %
    \centering
    \includegraphics[width=0.49\linewidth]{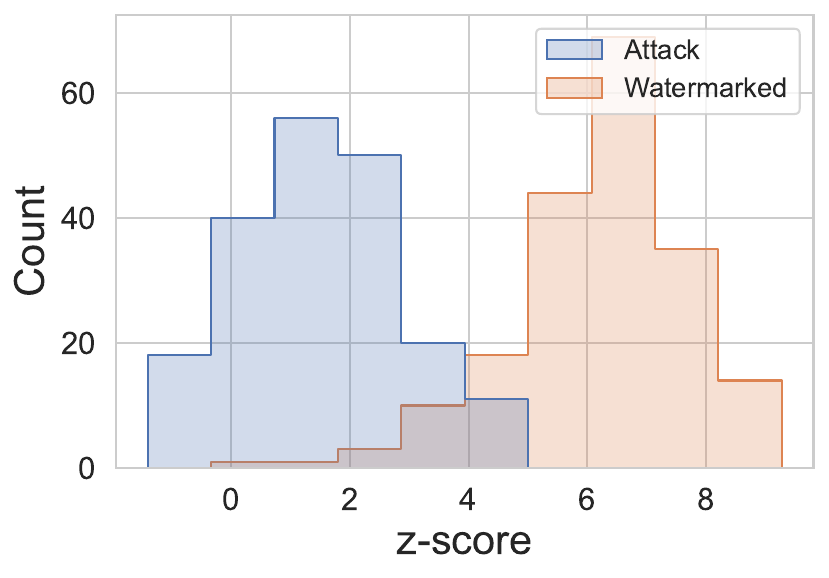} %
    \hfill
    \includegraphics[width=0.49\linewidth]{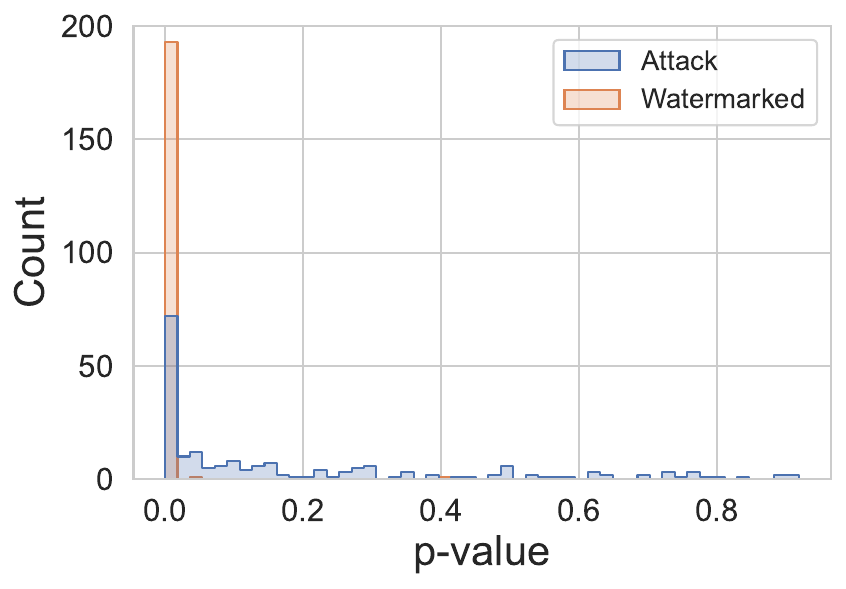} %
    \caption{Detection performance}
    \label{fig:z_hist200}\label{fig:p_hist200}
  \end{subfigure}%
  \hfill
  \begin{subfigure}[b]{0.315\textwidth} %
    \centering
    \includegraphics[width=\linewidth]{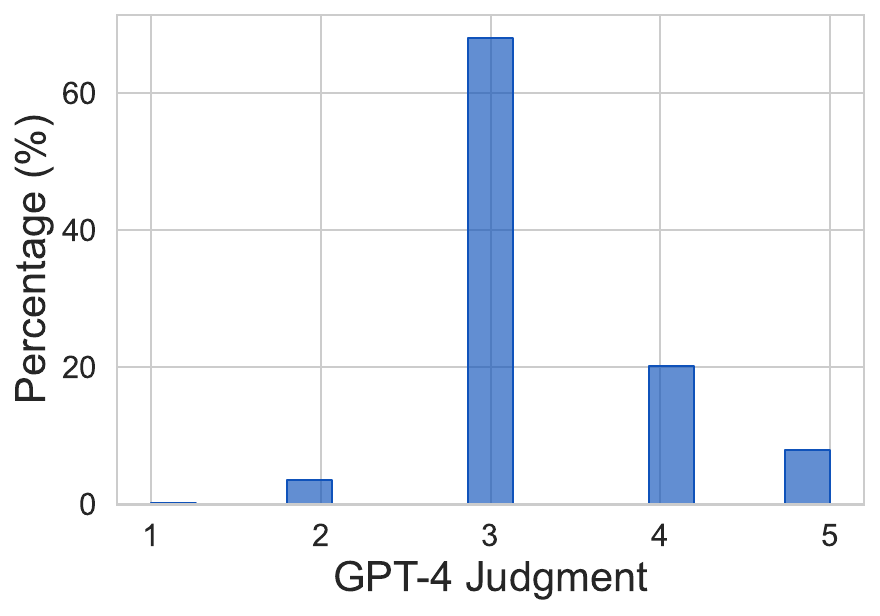} %
    \caption{GPT-4 judgment}
    \label{fig:j_hist200}
  \end{subfigure}
  \caption{
  (a) Detection performance before (Watermarked) and after (Attack) our attack using Llama2-7B with KGW \citep{kirchenbauer2023watermark}. 
  (b) Comparative evaluation on watermarked texts against texts post-attack using GPT-4 as a judge. Scoring criteria: 1=post-attack response is much better than watermarked one, 2=slightly better, 3=of similar quality, 4=slightly worse, 5=much worse. Each example is included as two data points, one for each ordering of the two outputs in the GPT-4 query.}
  \label{fig:hist}
\end{figure}

\paragraph{Detection performance and text quality over the course of the random walk.}
We also study how the watermark z-score and GPT-4 judge quality score change over time during the attack.
For the sake of efficiency, we ran these tests for only $12$ randomly selected prompts from the dataset. The results are averaged over these 12 examples, and the GPT-4 judge results are averaged over the two output orderings as well.
In~\Cref{fig:time}, as the number of traversal steps increases, the z-score steadily decreases while the average GPT-4 judgment score fluctuates around its initial value of 0. 
In~\Cref{fig:intermediate} of the appendix, we showcase how the text, detection statistics, and GPT-4 judge score change over the course of the random walk attack for a single example.

\paragraph{Controlling for perturbation oracle quality.}

If the perturbation oracle (T5-XL v1.1) were strong enough to produce texts on its own that are comparable in quality to the watermarked model $\model$ (Llama-2-7B), then our attack would be trivial.
As a baseline and to calibrate our GPT-4 judge, we generate 40 samples by asking T5 for text completion iteratively: 
we feed the C4 news prefix for text completion and concatenate the generated contents back as input until the length can match the watermarked response. 
Then we ask GPT-4 to compare the watermarked responses against the T5-generated counterparts using the quality judge prompt mentioned above. The results are stark: the judge strongly prefers the watermarked Llama-2-7b output over the T5 output for every single example. 
In other words, when the watermarked response is presented first, the judge always says it is ``much better'' than the T5 response; and when the T5 response is presented first, the judge always declares it ``much worse''. 
This indicates that the quality oracle component of the attack was crucial to its success.

\paragraph{The impact of response length on attack performance.}
We plot the average z-score as a function of T denoted as the token length of the generated text (\Cref{fig:len}). Note that a length greater than 400 is much larger than the default settings KGW was tested on, and we observe that the z-score keeps increasing for longer sequences, indicating detection performance degradation. On the other hand, our attack plateaus, showing that we can remove the watermark for long texts which can be essential for many practical tasks such as long-form reasoning \citep{nye2021show}, essay writing, etc.

\begin{figure*}[h]
\begin{minipage}[h]{\textwidth}
\centering
\includegraphics[width=.48\textwidth]{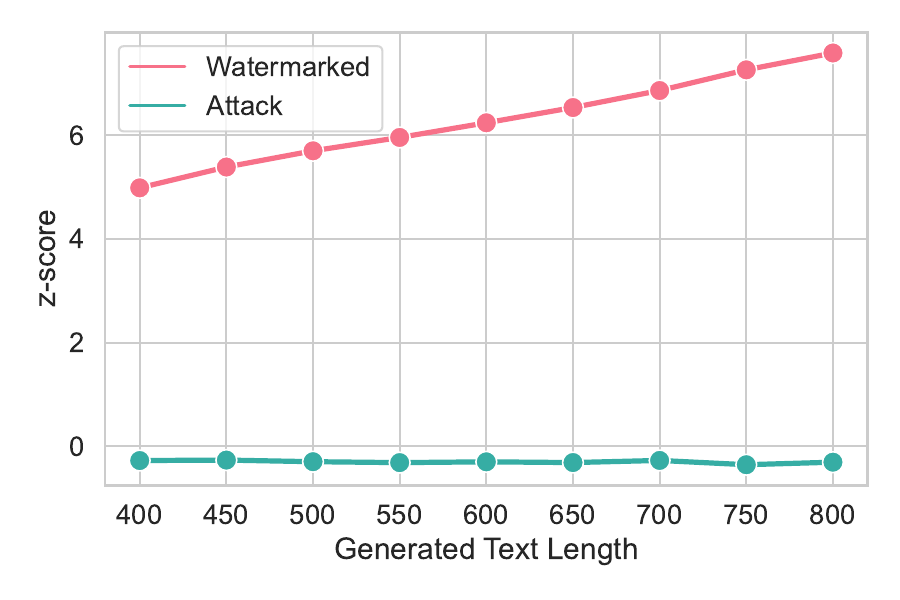}
\label{fig:z}
\includegraphics[width=.48\textwidth]{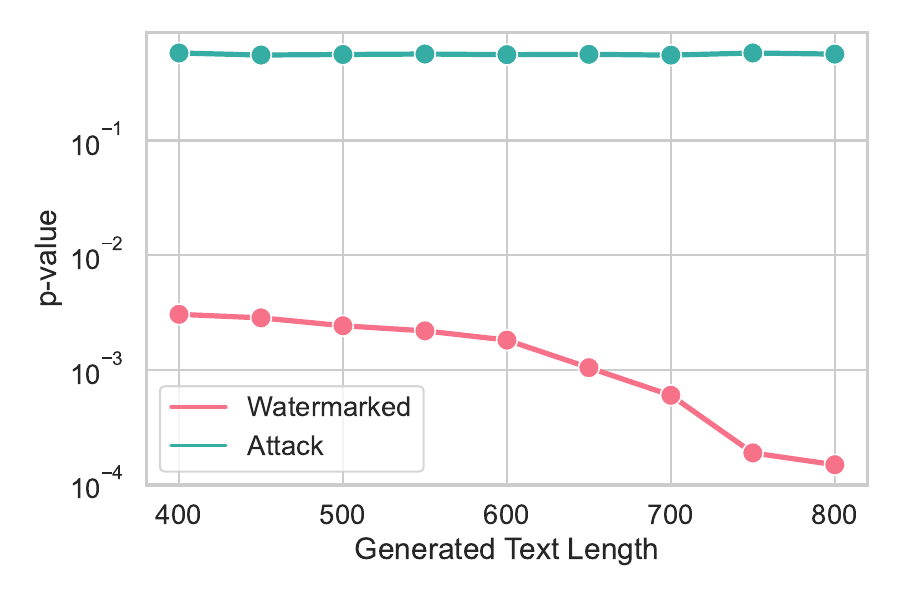} 
\label{fig:p} 

\end{minipage}
\caption{
Detection performance and w.r.t. the watermarked text length using Llama2-7B-Chat with KGW \citep{kirchenbauer2023watermark}. Results are aggregated across hundreds of examples.  %
}
\label{fig:len}
\end{figure*}

\subsection{Experimental Results on Vision-Language Models}
\vspace{-1mm}

\begin{figure*}[h]
   \centering
   \begin{subfigure}[b]{0.49\textwidth}
       \centering
       \includegraphics[width=0.49\textwidth]{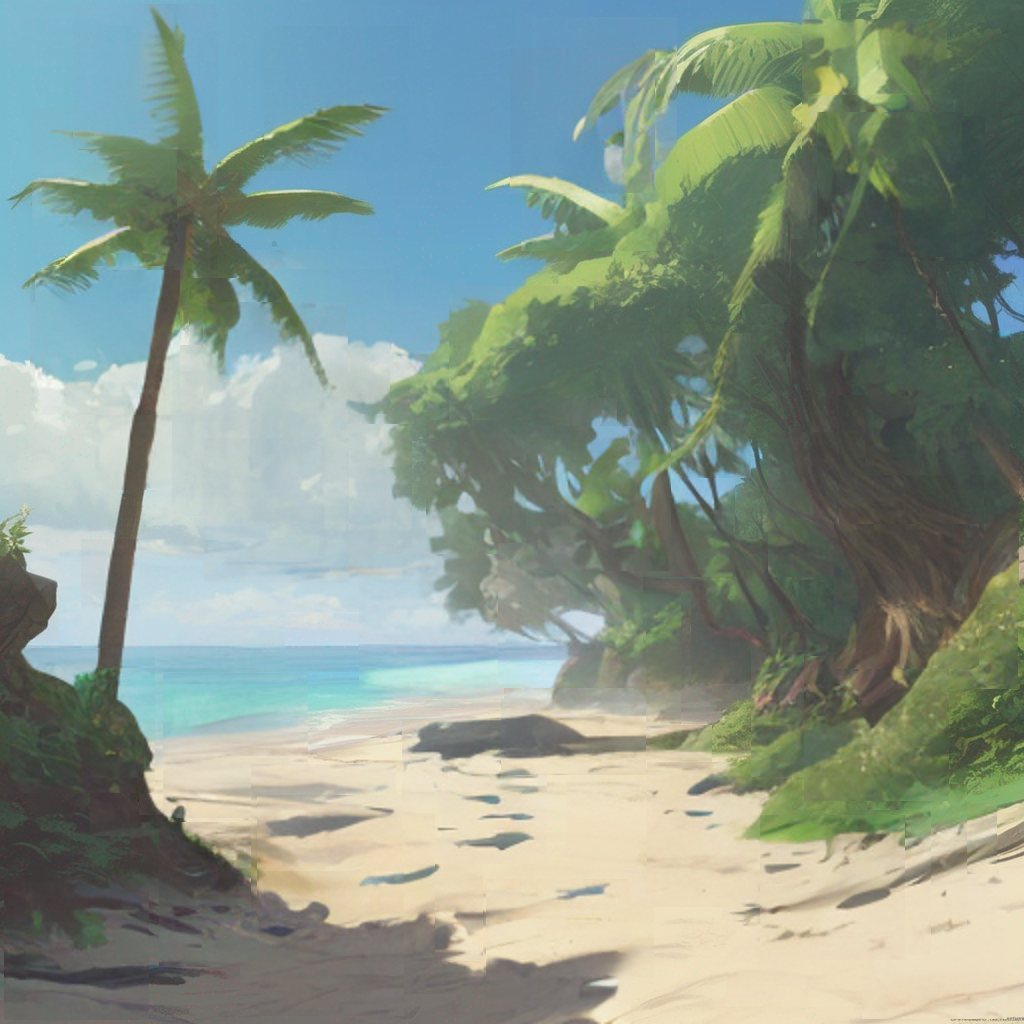}
       \includegraphics[width=0.49\textwidth]{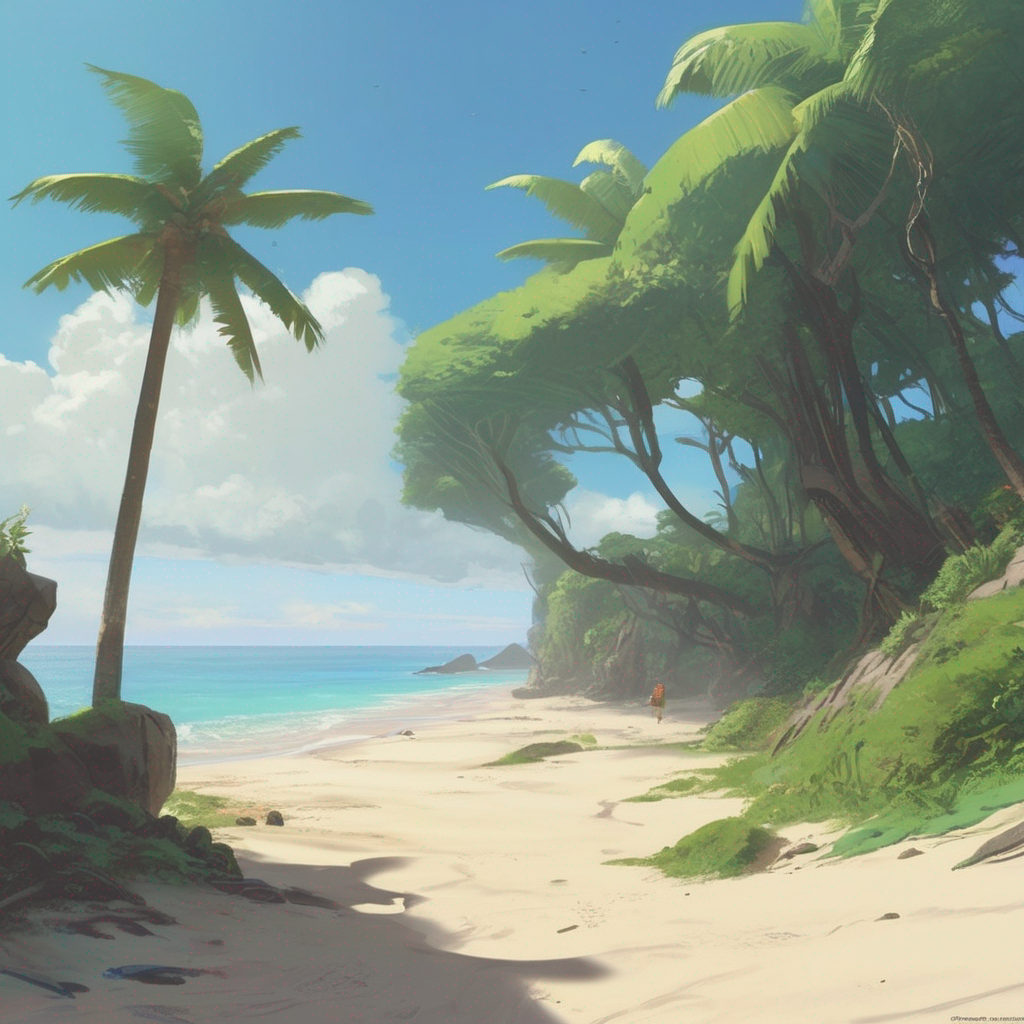}
       \caption{Invisible Watermark \citep{shield2023ivw}}
   \end{subfigure}
   \begin{subfigure}[b]{0.49\textwidth}
       \centering
       \includegraphics[width=0.49\textwidth]{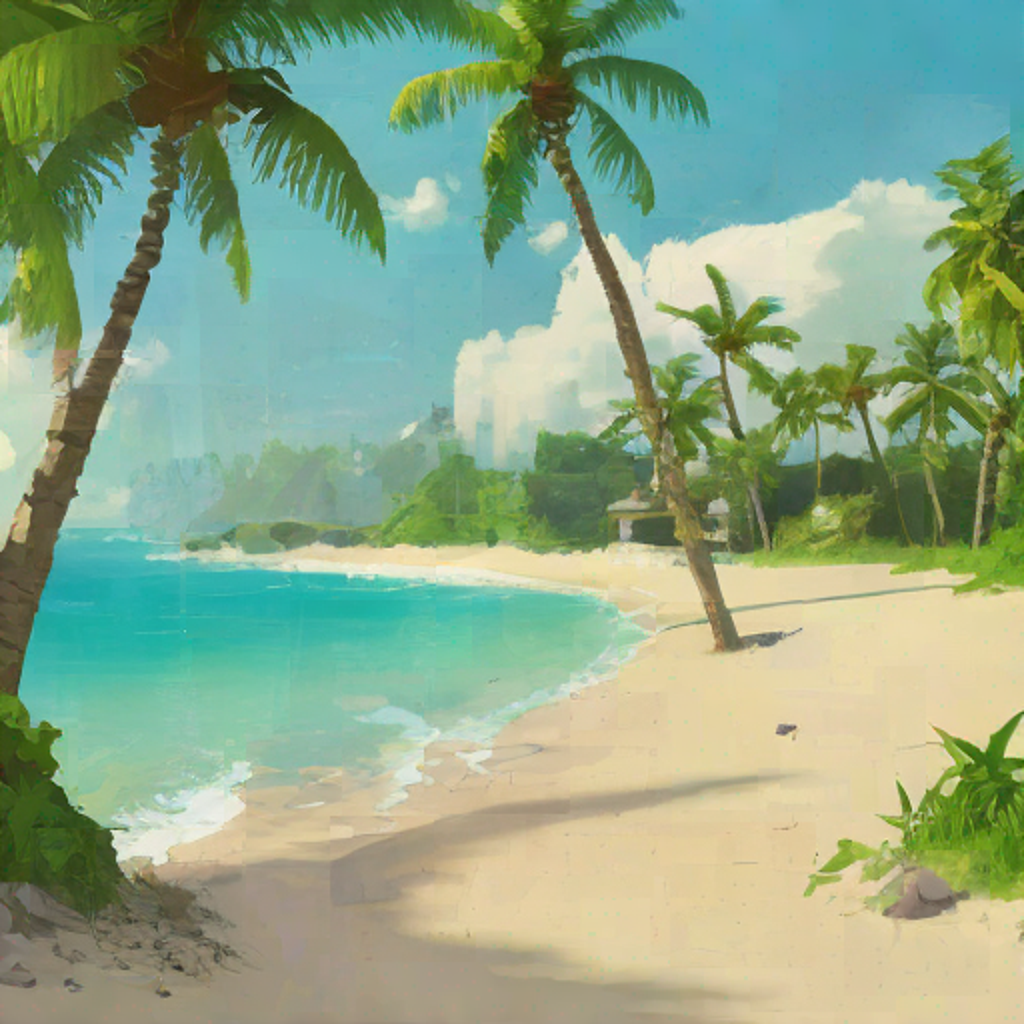}
       \includegraphics[width=0.49\textwidth]{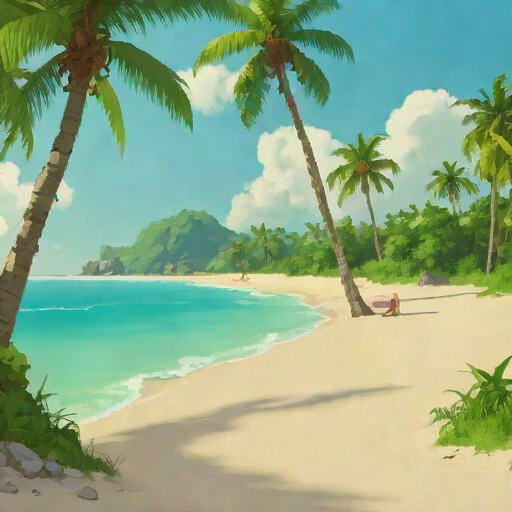}
       \caption{Stable Signature \citep{fernandez2023stable}}
   \end{subfigure}
   \caption{Qualitative examples of the watermarked images after (\textcolor{candypink}{left}) and before (\textcolor{caribbeangreen}{right}) our attack for two watermarking schemes. Images are generated by prompting stable-diffusion-2-base with the prompt \textit{``A long and winding beach, tropical, bright, simple, by Studio Ghibli and Greg Rutkowski, artstation\textbackslash n''}.
   Detection and quality evaluation results: Invisible Watermark (p-value $3.5e\text{-}15 \rightarrow 0.2354$, CLIP score $34.82 \rightarrow 33.60$, GPT-4 Judge $0$), Stable Signature (p-value $1.3e\text{-}5 \rightarrow 0.468$, CLIP score $32.27 \rightarrow 31.58$ , GPT-4 Judge $0$).
   }
   \label{fig:cv_qual}
\end{figure*}

\begin{table*}[h]
\caption{Average results on two image watermark schemes before and after attack, as applied to the stable-diffusion-xl. We adopt the same evaluation protocol as on LMs (\Cref{tab:main_tab}).} %
\centering
\resizebox{0.8\textwidth}{!}{
\begin{tabular}{cccc}
\toprule \toprule
Framework  & p-value & CLIP score & GPT-4 Judge \\ \midrule
 Stable Signature  \citep{fernandez2023stable} &   5.940e-6 $\rightarrow$ 0.059 &  33.91 $\rightarrow$ 33.40 &  -0.088 \\
 Invisible Watermark \citep{kirchenbauer2023watermark} &  1.793e-5 $\rightarrow$  0.206 & 35.64 $\rightarrow$ 35.51 &  - 0.138 \\

\bottomrule \bottomrule  
\end{tabular}
}
\label{tab:img_tab}
\end{table*}

As a proof of concept of the generality of our attack, we show that a version of the attack can remove watermarks that have been used for image diffusion models.
Specifically, we attack Invisible Watermark \citep{shield2023ivw} applied to the diffusion model stable-diffusion-xl-base-1.0 \citep{podell2023sdxl} and the Stable Signature \citep{fernandez2023stable} to sdxl-turbo \citep{sauer2023adversarial}.

We first generate 200 images from randomly chosen prompts that are verified to be effective in generating high-quality images and then filter those examples with p-values greater than $0.001$. 
For the remaining $141$ valid examples (for Stable Signature, all are valid), we implement our perturbation oracle using stable-diffusion-2-base \citep{Rombach_2022_CVPR}, which inpaints the image masked by a random square mask of size $0.02$ times the image size. Note that this is a weaker model than the watermarked models, but it is not much weaker, so we see these results as preliminary.
Our quality oracle is implemented as a reward model trained on Human Preference Score v2 \citep{wu2023human}.
We perturb the image for 100 valid random walk steps.
Finally, the attacked images are evaluated by calculating the CLIPScore \citep{hessel2021clipscore} and querying the multimodal API of GPT-4 \citep{openai2023gpt4}, gpt-4-turbo, to report the final quality comparison scores using the prompt ``\textit{[Prompt], Response A: [Image A], Response B: [Image B], Compare which of the two above figures is a better response of higher-quality to the given prompt. Explain your reasoning step by step.}''

\vspace{-1mm}
We show in the \Cref{tab:img_tab} that our attack can successfully remove the watermarks with only slight degradation in image quality using the same evaluation protocol as on LMs.

\begin{figure}[h]
    \centering
    \includegraphics[width=.99\textwidth]{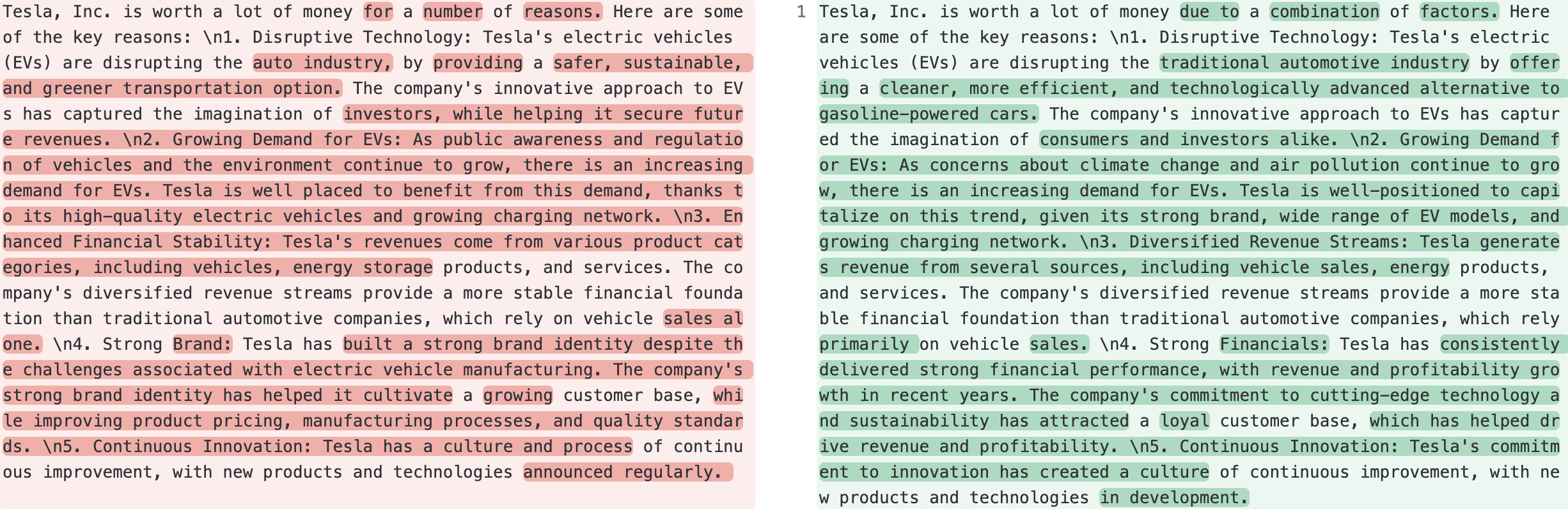} \\
    \vspace{2.5pt} %
    \includegraphics[width=.99\textwidth]{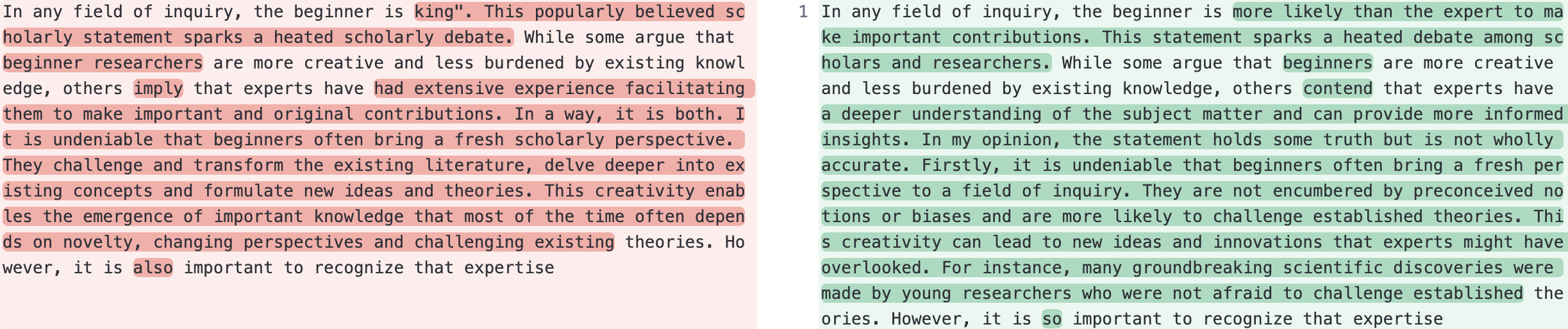}
    \caption{Un-watermarked example after our attack (\textcolor{candypink}{red}) and the corresponding original watermarked text (\textcolor{caribbeangreen}{green}) for the query ``\textit{Q: How come Tesla is worth so much money?}'' and an SAT essay writing prompt.
    We do a word-by-word comparison and the differences between the two texts are highlighted.
    }
    \label{fig:qual}
\end{figure}
\subsection{Qualitative Examples}
In \Cref{fig:qual} we display two concrete non-cherrypicked before-and-after examples of the effects of our random walk attack on model outputs, so that the reader can get a sense of how quality is affected by the process. We use one prompt from LFQA \cite{fan2019eli5} and one SAT essay prompt.
(For the SAT query, we use ``\textit{In any field of inquiry, the beginner is more likely than the expert to make important contributions. Write a response in which you discuss the extent to which you agree or disagree with the statement and explain your reasoning for the position you take. In developing and supporting your position, you should consider ways in which the statement might or might not hold true and explain how these considerations shape your position.}''). 
The comparison shows that the text after our attack can still be coherent, fluent, and on-topic. 
Such high-quality non-watermarked examples are abundant in our results. See \Cref{app:intermediate} for an example of the text at many intermediate rounds during the attack, with corresponding detection results and quality judge evaluations.
Moreover, \Cref{fig:cv_qual} showcases two set of images before and after our attack. 
Some image details such as the background, and shapes of objects get perturbed but the overall image can still fit the prompt provided.
We again see the p-values dramatically increase after our attack though with slight quality degradation.

\vspace{-2mm}
\section{Concluding Remarks}
\vspace{-1mm}
This work provides general impossibility results for strong watermarking schemes. Although our specific attack is not very efficient, its primary advantage lies in its generality. Furthermore, while our implementation targets language generation models, the concepts of quality and perturbation oracles apply to all generative contexts, including audio-visual data. We posit, without formal proof, that verification is simpler than generation; thus, a general enhancement in capabilities and flexibility is likely to benefit the attacker (i.e., through improved quality and perturbation oracles) more than the defender (i.e., through better planting and detection algorithms).
Hence we are not optimistic about the feasibility of strong watermarking schemes that would prevent a determined attacker from removing the watermark while preserving quality.
Nonetheless, \emph{weak} watermarking schemes do exist and can be useful, especially for safeguarding against unintentional leaks of AI-generated data (e.g., into training sets) or in scenarios with minimal risk that assume a low-effort attacker.
Moreover, watermarking is just one tool in the arsenal of safety efforts for generative AI modeling.
Cryptography can also be used to establish provenance of data.
Ultimately, in many instances of disinformation, it is crucial to verify the data's source rather than to confirm whether AI generated it.
Thus our goal is to set practical expectations regarding what watermarking schemes can achieve, thereby contributing to the safer use of AI models.

\vspace{-2mm}
\section*{Broader Impact}
\vspace{-2mm}
We believe that investigating the possibilities of watermarking schemes at this stage can help to provide a better understanding of the inherent tradeoffs and give policymakers realistic expectations of what watermarking can and cannot provide. While our techniques can be used to remove watermarks from existing schemes, they are not the most efficient way to do so, with the benefit being generality rather than efficiency. Moreover, our implementation is for text generation models, while currently widely deployed watermarks are for image generation models. While it is possible to adapt our ideas to attack deployed image generative models, we do not provide a recipe for doing so in this paper. Thus, our work isn’t likely to be used by malicious actors. Rather, we see exposing fundamental weaknesses in the watermarking paradigm as a contribution to the ongoing discussion on how to mitigate the misuse of generative models. We hope our findings will be taken into account by organizations building generative models and policymakers regulating them. 

\vspace{-2mm}
\section*{Acknowledgements}
\vspace{-2mm}
We thank John Thickstun and Tom Goldstein for helpful discussions. Kempner Institute computing resources enabled this work. Hanlin Zhang is supported by an Eric and Susan Dunn Graduate Fellowship. Boaz Barak acknowledges funding supported by the Kempner Institute, Simons Investigator Fellowship, NSF grant DMS-2134157, DARPA grant W911NF2010021, and DOE grant DE-SC0022199. Danilo Francati was supported by the Carlsberg Foundation under the Semper Ardens Research Project CF18-112 (BCM). Ben Edelman acknowledges funding from the National Science Foundation Graduate Research Fellowship Program under award DGE-214074. Daniele Venturi was supported by project SERICS (PE00000014) under the MUR National Recovery and Resilience Plan funded by the European Union - NextGenerationEU, and by Sapienza University under the project SPECTRA.
Giuseppe Ateniese was supported in part by grants from the Commonwealth Cyber Initiative (CCI) and the Commonwealth Commercialization Fund (CCF), as well as corporate gifts from Accenture, Lockheed Martin Integrated Systems, and Protocol Labs.

\bibliography{ref.bib} 
\bibliographystyle{plainnat}

\newpage
\appendix
\addcontentsline{toc}{section}{Appendix} %
\renewcommand \thepart{} %
\renewcommand \partname{}
\part{\Large{\centerline{Appendices}}}
\parttoc
\newpage

\section{Questions \& Answering}
\textbf{Perturbation oracle only checks the quality of the output but does not guarantee content preservation.}

We believe that our theoretical definition of attack as quality preserving is the correct one for both images and language models. For example, in the context of disinformation, if an attacker generates an output with respect to a prompt X (e.g., ``An image of a person doing X'' or ``A New York Times story that X happened'') then they only care about the output's fitness to the prompt. 
So, if they are able to modify the output to erase the watermark while preserving quality (i.e. fitness to prompt) this would be a successful attack, even if some semantic details are changed. Therefore, we think that our metric of looking at quality (by a GPT4 judge, which crucially is not used in the attack itself) is the correct way to evaluate success.

A common question is to explain the difference between ours and prior attacks, which attempted to modify the output to be semantically equivalent while erasing the watermark. 

One of the contributions of this paper is the formulation of the security of watermarking with respect to a measure of quality which is \textbf{fitness to the prompt}, as opposed to a measure of ``similarity''. 
This formulation captures more closely what the attacker cares about: if they request an output Y to a prompt X from a model, and want to perturb it to Y’ that is not watermarked, then they don’t care whether Y’ is semantically identical to X but whether it has as good a fitness with the prompt. This implies that we are supposed to evaluate text quality conditioned on the prompt as opposed to metrics like perplexity that only evaluate the response.

Second, this formulation is necessary for the impossibility result to hold in generality. For example, if the prompt is to write a story, then the model would have the freedom to choose semantic details (e.g., names of characters, physical descriptions) that can be used to embed the watermarking signal. It may be impossible to remove the watermark without modifying these semantic details. Note that if the prompt did specify these details such as names of characters, then neither the model would have the freedom to modify them to inject the watermark nor the attacker could change them to erase it. Hence we believe fitness to the prompt (aka quality) is the right measure for a watermark attack.

Since our attack has a different objective (quality preservation as opposed to semantic similarity), it is also qualitatively different from prior works \citep{sadasivan2023can, saberi2023robustness}. We may consider a similar space of perturbations, but our algorithm of whether to accept or reject a given perturbation is fundamentally different.

\textbf{What are the tradeoffs between type-I and type-II errors as the output space entropy changes?}

We formalized only the minimal properties required to demonstrate our impossibility result which does not include entropy, type-I and type-II errors. Setting aside the graph's properties, our impossibility result applies to watermarking schemes with a small false positive rate $\epsilon_{pos}$ (please refer to Def. 4), which is fundamental for having a functional watermarking scheme. Indeed, small $\epsilon_{pos}$ guarantees that, with (high) probability $1-\epsilon_{pos}$, the scheme classifies as un-watermarked responses computed independently from the secret key. This type of response corresponds to human-generated content. This is a necessary property for having a functional watermarking scheme that permits distinguishing between machine-generated and human-generated content (for more details see also page 5, right after Def. 4).

To summarize, our impossibility result applies to any good enough watermarking scheme independently from the output space entropy amount, type-I, and type-II errors in detecting the watermark. This only makes our impossibility stronger and justifies the absence of entropy/type-I/type-II formalizations.

\section{Mathematical Background and Random Walks}\label{sec:preliminaries}

\subsection{Notation}\label{sec:notation}
We use the notation $[n] = \{1, 2, \ldots, n\}$.
Capital bold-face letters (such as $\rv{X}$) are used to denote random variables, small letters (such as $x$) to denote concrete values, calligraphic letters (such as $\cX$) to denote sets, and serif letters (such as $\mathsf{A}$) to denote algorithms.
For a string $x \in \bin$, we let $|x|$ be its length; if $\cX$ is a set, $|\cX|$ represents the cardinality of $\cX$.

We denote with $\vector{x} \in \cX^n$ the column vector of length $n$ with elements from $\cX$. Similarly, we denote with $\matrix{X} \in \cX^{n\times m}$ the matrix with $n$ rows, $m$ columns, and elements from $\cX$.
We write $\vector{x}^{\top}$ (resp. $\matrix{X}^{\top}$) to denote the transpose vector of $\vector{x}$ (resp. the transpose matrix of $\matrix{X}$).
Given a vector $\vector{x} \in \cX^n$ (resp. a matrix $\matrix{X} \in \cX^{n\times m}$), we write $\vector{x}(i) \in \cX$ (resp. $\matrix{X}(i,j)$) to denote the $i$-th element of $\vector{x}$ (resp. the element located in the $i$-th row and $j$-th column of $\matrix{X}$).
We say a vector $\vector{x} \in \RR^n$ is a (probability) distribution if $\sum_{i}\vector{x}(i) = 1$; if $\vector{x} \in \RR^n$ is a distribution, we use the notation $\vector{x}$ and $\rv{x}$ interchangeably (where $\vector{x}(i) = \rv{x}(i)$ for every $i \in [n]$).
When $x$ is chosen uniformly from a set $\cX$, we write $x \getsr \cX$; if $\rv{x}$ is a distribution, we write $x \getsr \rv{x}$ to denote the act of sampling $x$ according to the distribution $\rv{x}$. That is, $\Prob{x = i : i\getsr \rv{x}} = \rv{x}(i)$.

\subsection{Graphs, Random Walks, and Mixing Time}\label{sec:graphs}
Next, we focus on \emph{weighted directed} graphs $\graph = (\cV,\cE)$ composed of $n$ vertices $\cV = [n]$.
We denote by $\weight(i,j)$ the weight of the edge $(i,j)\in\cE$ from $i \in \cV$ to $j \in \cV$.
We assume that $\weight(i,j) = 0$ if and only if $(i,j) \not\in \cE$. Moreover, we define $\weight(i,\star) = \sum_j \weight(i,j)$, i.e., $\weight(i,\star)$ represents the sum of the edges' weights with source vertex $i$.

\paragraph{Random Walks.}
A $t$-step random walk over $\graph$ is a probabilistic process such that, at each step, a neighbor is selected from the neighbors of the current vertex, according to the weight distribution.
The transition matrix $\matrix{P} \in \RR^{n\times n}$ of a random walk over a weighted directed graph $\graph = (\cV,\cE)$ is defined as $\matrix{P}(i,j) = \frac{\weight(i,j)}{\weight(i,\star)}$, i.e., $\matrix{P}(i,j)$ represents the probability (based on the graph's weights) of reaching vertex $j$ from vertex $i$.
Let $\rv{p}_0\in\RR^{n}$ be the starting distribution (of a random walk), and let $\rv{p}_t\in\RR^{n}$ represent the distribution after $t$ steps of the random walk.
By definition, we have $\rv{p}_{t+1}(i) = \sum_{j:(j,i) \in \cE} \rv{p}_t(j)\cdot \matrix{P}(j,i)$, where $\rv{p}(i)$ is the probability associated with $i\in \cV$.
Similarly, the (global) distribution after a $t$-step random walk can be expressed as $\rv{p}_{t+1} = \matrix{P}^{\top}\cdot \rv{p}_t$.

Below, we recall the definition of the stationary distribution of a random walk.
\begin{definition}[Stationary distribution of random walks]\label{def:stationary-distr}
    Let $\graph= (\cV,\cE)$ be a weighted directed graph and $\matrix{P}$ be the transition matrix of $\graph$.
    We say that $\vector{\pi} \in \RR^{n}$ is a {\em stationary distribution} for $\matrix{P}$ if $\matrix{P}^{\top}\cdot \vector{\pi} = \vector{\pi}$.
\end{definition}
\noindent In other words, the above definition states that the probability distribution after a $1$-step random walk is the stationary distribution $\vector{\pi}$ if the starting distribution $\rv{p}_0$ is the stationary distribution itself.
By induction, this also implies that the distribution $\rv{p}_{t}$ after a $t$-step random walk is equal to the stationary distribution $\vector{\pi}$ if $\rv{p}_0 = \vector{\pi}$, for every $t \geq 1$.

The following theorem provides the stationary distribution for any weighted directed graph.
\begin{theorem}\label{def:stationary-distribution-graphs}
    Let $\graph=(\cV,\cE)$ be a weighted directed graph and $\matrix{P}$ be its transition matrix.
    The distribution $\vector{\pi} \in \RR^{n}$ such that 
    \[
        \forall i \in \cV,\ \vector{\pi}(i) = \frac{\weight(i,\star)}{\sum_{i \in \cV} \weight(i,\star)} = \frac{\sum_{j\in \cV} \weight(i,j)}{\sum_{i \in \cV} \sum_{j \in \cV} \weight(i,j)} = \frac{\sum_{j\in \cV} \weight(i,j)}{\sum_{(i,j) \in \cE} \weight(i,j)}
    \]
    is a stationary distribution for $\matrix{P}$.
\end{theorem}

Interestingly, the stationary distribution is unique when the underlying directed graph $\graph$ is irreducible (i.e., the graph does not have leaf nodes).
Moreover, if $\graph$ is also aperiodic then, independently of the initial starting distribution $\rv{p}_0$, a random walk converges to its stationary distribution $\vector{\pi}$ as $t \rightarrow \infty$.
\begin{theorem}[Convergence to the stationary distribution]\label{def:uniqueness-stationary-distribution}
    If a weighted directed graph $\graph = (\cV,\cE)$ is irreducible and aperiodic, there exists a unique stationary distribution $\vector{\pi}$.
    Moreover, for every $\rv{p}_0 \in \RR^{n}$, $\rv{p}_t = (\matrix{P}^\top)^t\cdot \rv{p}_0$ (i.e., the probability distribution after a $t$-step random walk with starting distribution $\rv{p}_0$) converges to $\vector\pi$ as $t \rightarrow \infty$.
\end{theorem}

In this work, we are interested in setting a bound for $t$ in the above theorem.
In more detail, we want to bound the minimum number of steps required by the random walk to get close enough to its stationary distribution.
This is known as the $\epsilon_{\sf dist}$-mixing time, and it can be bounded using the {\em second largest eigenvalue} (in absolute value) of the transition matrix of the graph (note that the second largest eigenvalue also has connections with the conductance of the graph).
Below, we report the formal definition of mixing time and its corresponding bound.
\begin{definition}[$\epsilon_{\sf dist}$-mixing time]\label{def:mixing-time}
    Let $\graph = (\cV,\cE)$ be a weighted directed graph that is irreducible and aperiodic (as in~\Cref{def:uniqueness-stationary-distribution}) and let $\matrix{P}$ be its corresponding transition matrix.
    For any $0 < \epsilon_{\sf dist} \leq 1$, the {\em $\epsilon_{\sf dist}$-mixing time} $t_{\sf min}(\epsilon_{\sf dist})$ of $\matrix{P}$ is the smallest number of steps $t$ such that for every starting distribution $\rv{p}_0 \in \RR^{n}$, we have
    \[
        \Big\vert \rv{p}_t - \vector\pi \Big\vert = \Big\vert (\matrix{P}^{\top})^{t}\cdot \rv{p}_0 - \vector\pi \Big\vert \leq \epsilon_{\sf dist},
    \]
    where $\vector{\pi}$ is the unique stationary distribution of $\matrix{P}$.
\end{definition}

\begin{theorem}[Bound on $\epsilon_{\sf dist}$-mixing time]\label{thm:mixing-time}
    Let $\graph = (\cV,\cE)$ be a weighted directed graph that is irreducible and aperiodic, and let $\matrix{P}$ be its corresponding transition matrix.
    Also, let $\eigenvalue_1\geq \eigenvalue_2 \geq \ldots \geq \eigenvalue_n$ be the eigenvalues of $\matrix{P}$, and let $\vector{\pi}$ be the unique stationary distribution of $\matrix{P}$ (\Cref{def:uniqueness-stationary-distribution}).
    For $g = \max \{|\eigenvalue_2|,\ldots, |\eigenvalue_{n}|\}$ and $\pi_{\sf \min} = \min\{\vector{\pi}(1),\ldots, \vector{\pi}(n)\}$, the $\epsilon_{\sf dist}$-mixing time of $\matrix{P}$ is 
    \[
        t_{\sf \min}(\epsilon_{\sf dist}) \leq O\left(\frac{1}{1-g}\cdot \log\left(\frac{1}{\pi_{\sf \min}\cdot \epsilon_{\sf dist}}\right)\right).
    \]
\end{theorem}

\section{Extended Impossibility Theorem}\label{app:extended-impossibility-result}

In this section, we include an extended version of the impossibility result which explicates the concrete relation between the adversarial advantage $\epsilon$, the number of queries, and the $\epsilon_{\sf pert}$-preservation of $\perturbationOracle$. 
In addition, to make the result as generic as possible, we define $\perturbationOracle$'s properties (i.e., irreducibility and aperiodicity) w.r.t. any arbitrary $v$-th quality percentile (for $v \in [0,100]$) of the quality values of all possible responses (of prompt $x$) that can be obtained by watermarking $\model \in \cM$ (recall that in~\Cref{thm:impossibility-simple} we used the median instead of the $v$-th quality percentile).

Formally, for $v \in [0,100]$, let $\cQ_{\model,x} = \{q_1,q_2,\ldots\}$ and $q_{\sf min}$ be defined as follows:
\begin{align}
    \cQ_{\model,x} & \text{ is defined as in~\Cref{eq:quality-set}},\label{eq:quality-set-pecentile} \\
    q_{\sf min} &= \underset{\model\in\cM,x\in\cX}{\min}\{q_{\model,x}\} \text{ where } q_{\model,x} \text{ is the $v$-th quality percentile of } \cQ_{\model,x}.\label{eq:minimum-median-percentile}
\end{align}
Below, we report the extended version of the impossibility result whose perturbation oracle is defined w.r.t. $q_{\sf min}$ of~\Cref{eq:minimum-median-percentile}. The proof is given in~\Cref{sec:proof}.

\begin{theorem}\label{thm:impossibility-perturbation}
    Let $\Pi=(\watermark,\detect)$ be a watermarking scheme for a class of generative models $\cM = \classModel$ with an associated quality function $\quality:\cX\times\cY \rightarrow [0,1]$.
    Let $\perturbationOracle:\cX\times\cY \rightarrow \cY$ be a perturbation oracle (defined over the same prompt space $\cX$ and output space $\cY$ of the class $\cM$) with the same associated quality function $\quality:\cX\times\cY \rightarrow [0,1]$ as $\Pi$.

    Under the following conditions
    \begin{enumerate}
        \item The watermarking scheme $\Pi$ has a false positive $\epsilon_{\sf pos}$-rate (\Cref{def:false-positive-negative});\label{itm:watermarking}
        \item The perturbation oracle $\perturbationOracle$ is $\epsilon_{\sf pert}$-preserving (\Cref{def:perturbation-oracle});\label{itm:oracle}

        \item For every non-watermarked model $\model \in \cM$, for every prompt $x\in\cX$, for every quality $q\in[q_{\sf min},1]$, the $q$-quality $x$-prompt graph representation $\graph^{\geq q}_x$ of $\perturbationOracle$ is irreducible and aperiodic where $q_{\sf min}$ is the minimum median defined in~\Cref{eq:minimum-median-percentile} (for some arbitrary $v\in[0,100]$).
        Also, let $\vector{\pi}_{x,q}$ be the unique stationary distribution\footnote{Recall that a random walk converges to its unique stationary distribution when the corresponding weighted directed graph is irreducible and aperiodic (\Cref{def:uniqueness-stationary-distribution}).} of the transition matrix $\matrix{P}_{x,q}$ of $\graph^{\geq q}_{x}$ (\Cref{def:stationary-distr}) and, for $\epsilon_{\sf dist} \in [0,1]$, let $t_{x,q}$ be the $\epsilon_{\sf dist}$-mixing time of $\matrix{P}_{x,q}$ (\Cref{def:mixing-time,thm:mixing-time}) defined as follows:
        \[
            t_{x,q} = \omega\left(\frac{1}{1- \max\{|\eigenvalue_{2}^{(x,q)}|,\ldots,|\eigenvalue_{n}^{(x,q)}|\}}\cdot \log\left(\frac{1}{\pi^{(x,q)}_{\sf \min}\cdot \epsilon_{\sf dist}}\right)\right)
        \] 
        where $\pi^{(x,q)}_{\sf \min} = \min\{\vector{\pi}_{x,q}(1),\ldots, \vector{\pi}_{x,q}(n_{x,q})\}$, $n_{x,q} = |\cV^{\geq q}_x|$, and $\eigenvalue_{1}^{(x,q)} \geq \eigenvalue_{2}^{(x,q)} \geq \ldots \geq \eigenvalue_{n}^{(x,q)}$ are the eigenvalues of the transition matrix $\matrix{P}_{x,q}$;\footnote{Observe that $t_{x,q}$ is asymptotically larger than the $\epsilon_{\sf dist}$-mixing time of $\graph^{\geq q}_x$ as defined in~\Cref{thm:mixing-time}.}\label{itm:graph-assumption}
    \end{enumerate}
    there exists an \emph{oracle-aided} universal adversary $\adversary^{\perturbationOracle(\cdot,\cdot),\quality(\cdot,\cdot)}$ that $\epsilon$-breaks $\Pi$ (\Cref{def:erasure-attack}) by submitting at most $t$ queries to $\perturbationOracle$ where 
    \begin{align*}
        \epsilon & = \left(1-\frac{v}{100}\right)\left(1-\epsilon_{\sf pos}\right)\left(1-\epsilon_{\sf dist}\right)\left(1 - \sum^{t - t_{\sf err}-1}_{k = 0} {t\choose k} \left(\epsilon_{\sf pert}\right)^{k}\left(1-\epsilon_{\sf pert}\right)^{t-k}\right), \text{ and}\\
        t & = \max_{x\in\cX, q \in [q_{\sf min},1]}\{t_{x,q}\} + t_{\sf err}.
    \end{align*}
\end{theorem}
\noindent Observe that the above theorem implies~\Cref{thm:impossibility-simple} since $q_{\sf min}$ is the minimum median quality (as in~\Cref{thm:impossibility-simple}) when $v=50$.

For sufficiently small values of $v$, $\epsilon_{\sf dist}$, and sufficiently high values of $\epsilon_{\sf pert}$, the adversarial advantage $\epsilon$ of~\Cref{thm:impossibility-perturbation} approaches $1-\epos$ (i.e., the best possible advantage).\footnote{If $\epsilon_{\sf pert}$ is not high enough, the same result can be achieved by increasing $t_{\sf err}$ of~\Cref{thm:impossibility-perturbation}. Note that by increasing $t_{\sf err}$, we increase the overall number $t$ of perturbation queries.}
In turn, if $\epos$ is also small (e.g., $\epos \leq \frac{1}{10}$), $\epsilon$ is close to $1$.

\subsection{Proof of~\Cref{thm:impossibility-perturbation}}\label{sec:proof}
Fix $\perturbationOracle$ and $\quality$ as defined in~\Cref{thm:impossibility-perturbation}.
Let $\model \in \cM$ and let $(\k,\model_\k)$ (output by $\watermark(\model)$) be the original generative model and its watermarked version (with the associated secret-key $\k$) that the adversary is trying to break (recall that $\model$ and $(\k,\model_\k)$ are unknown to $\adversary$).
Consider the following adversary $\adversary^{\perturbationOracle(\cdot,\cdot),\quality(\cdot,\cdot)}$:
\begin{enumerate}
    \item Receive a prompt $x\in \cX$ and a watermarked output $y \in \cY$ (recall that $y$ is computed as $y \getsr \model_\k(x)$ as defined in the erasure attack experiment of~\Cref{def:erasure-attack}).
    \item Initialize $y_0 = y$ and $\ctr = 0$.
    \item Send $(x,y_0)$ to $\quality$ and receive the answer $q_0$.
    \item For each $i \in [t]$, the adversary proceeds as follows (where $t$ is as defined in~\Cref{thm:impossibility-perturbation}):
    \begin{enumerate}
        \item Send $(x,y_{i-1})$ to $\perturbationOracle$ and receive the answer $\tilde y$.
        \item Send $(x,\tilde y)$ to $\quality$ and receive the answer $\tilde q$.
        \item If $\tilde q \geq q_0$, set $y_i = \tilde y$ and increment the counter $\ctr$ (i.e., $\ctr = \ctr + 1$). Otherwise, if $\tilde q < q_0$, set $y_i = y_{i-1}$.
    \end{enumerate}
    \item Finally, output $y_t$ if $\ctr \geq t-t_{\sf err}$ (where $t$ and $t_{\sf err}$ are as defined in~\Cref{thm:impossibility-perturbation}). Otherwise, if $\ctr < t-t_{\sf err}$, output $\bot$ (i.e., an error message).
\end{enumerate}
\noindent To show that $\adversary^{\perturbationOracle(\cdot,\cdot),\quality(\cdot,\cdot)}$ $\epsilon$-breaks the watermarking scheme $\Pi$, we prove the following three lemmas.
\begin{lemma}\label{lmm:high-quality-watermark-output}
    For every $\model\in\cM$, for every prompt $x\in\cX$, we have 
    \[
        \prob{\quality(x,y_0) \geq q_{\sf min}} \geq 1-\frac{v}{100},
    \]
    where $y_0 = y$ is the watermarked output given as input to the adversary $\adversary^{\perturbationOracle(\cdot,\cdot),\quality(\cdot,\cdot)}$.
\end{lemma}
\begin{proof}
    The lemma follows by observing that $q_{\sf min} = \underset{\model\in\cM,x\in\cX}{\min}\{q_{\model,x}\}$ where $q_{\model,x}$ is the $v$-th quality percentile of the list $\cQ_{\model,x}$ as defined in~\Cref{eq:quality-set-pecentile,eq:minimum-median-percentile}.
    In other words, $q_{\sf min}$ is the minimum among the $v$-th percentiles $\{q_{\model,x}\}$ each calculated over all possible random coins of both $\watermark$ and the watermarked version of $\model$.
    By definition, this implies that 
        \begin{equation*}\label{eq:quality-median}
            \prob{\quality(x,\model(x)) \geq q_{\sf min}} \geq 1-\frac{v}{100}.
        \end{equation*}
    This concludes the proof of~\Cref{lmm:high-quality-watermark-output}.
\end{proof}
\begin{lemma}\label{lmm:binomial}
    For every $\model\in\cM$, for every prompt $x\in\cX$, we have that 
    \[
        \prob{ \adversary^{\perturbationOracle(\cdot,\cdot),\quality(\cdot,\cdot)}(x,y) \neq \bot} = 1 - \sum^{t - t_{\sf err}-1}_{k = 0} {t\choose k} \left(\epsilon_{\sf pert}\right)^{k}\left(1-\epsilon_{\sf pert}\right)^{t-k},
    \]
    where $y_0 = y$ is the watermarked output given as input to the adversary $\adversary^{\perturbationOracle(\cdot,\cdot),\quality(\cdot,\cdot)}$.
\end{lemma}
\begin{proof}
    Note that $\adversary^{\perturbationOracle(\cdot,\cdot),\quality(\cdot,\cdot)}(x,y)$ outputs $\bot$ only if $\ctr < t-t_{\sf err}$.
    Moreover, {\em the counter $\ctr$ is not incremented} only when the perturbation oracle, on input $(x,y_i)$ (for some $i \in [t]$), returns $\tilde y$ such that $\quality(x,\tilde y) < \quality(x,y_i)$.
    The latter occurs with probability at most $1-\epsilon_{\sf pert}$ since the perturbation oracle $\perturbationOracle$ is $\epsilon_{\sf pert}$-preserving (\Cref{def:perturbation-oracle}).
    
    Let $\rv{X}$ be the random variable describing the value of $\ctr$ at the end of the adversary's computation.
    Then, we have that
    \begin{equation}\label{eq:binomial-1}
        \prob{\adversary^{\perturbationOracle(\cdot,\cdot),\quality(\cdot,\cdot)}(x,y) \neq \bot} = \prob{\rv{X}\geq t-t_{\sf err} } = 1 - \prob{\rv{X} < t-t_{\sf err}} = 1 -  \prob{\rv{X} \leq t-t_{\sf err}-1}.
    \end{equation}
    The probability $\prob{\rv{X} \leq t - t_{\sf err} -1}$ is characterized by a binomial distribution where the probability of incrementing $\rv{X}$ (resp. not incrementing $\rv{X}$) is $\epsilon_{\sf pert}$ (resp. $1-\epsilon_{\sf pert}$). Formally,
    \begin{equation}\label{eq:binomial-2}
        \prob{\rv{X} \leq t-t_{\sf err}-1} = \sum^{t - t_{\sf err}-1}_{k = 0} {t\choose k} \left(\epsilon_{\sf pert}\right)^{k}\left(1-\epsilon_{\sf pert}\right)^{t-k}.
    \end{equation}
    \Cref{lmm:binomial} follows by combining~\Cref{eq:binomial-1,eq:binomial-2}.
\end{proof}
\begin{lemma}\label{lmm:stationary-distribution}
    For every $\model\in\cM$, for every prompt $x\in\cX$, conditioned on $\adversary^{\perturbationOracle(\cdot,\cdot),\quality(\cdot,\cdot)}(x,y) \neq \bot$ and $\quality(x,y) = q_0 \geq q_{\sf min}$, we have that
    \[
        \left\vert \adversary^{\perturbationOracle(\cdot,\cdot),\quality(\cdot,\cdot)}(x,y) - \vector{\pi}_{x,q_0} \right\vert \leq \epsilon_{\sf dist}
    \]
    where $\vector{\pi}_{x,q_0}$ is the unique stationary distribution of the transition matrix $\matrix{P}_{x,q_0}$ of $\graph^{\geq q_0}_x$.
\end{lemma}
\begin{proof}
    Assume that $\adversary^{\perturbationOracle(\cdot,\cdot),\quality(\cdot,\cdot)}(x,y) \neq \bot$ and $\quality(x,y) = q_0 \geq q_{\sf min}$.
    It is easy to see that the computation of the adversary $\adversary^{\perturbationOracle(\cdot,\cdot),\quality(\cdot,\cdot)}(x,y)$ is exactly a random walk over $\graph^{\geq q_0}_x = (\cV^{\geq q_0}_x,\cE^{\geq q_0}_x)$, where $q_0$ is the quality of the watermarked output $y = y_0$ given as input to the adversary. 
    This is because at each iteration $i \in [t]$, the adversary sets $y_i = \tilde y$ (i.e., it moves from $y_{i-1}$ to $y_i = \tilde y$ according to $\graph^{\geq q_0}_x$) only if $\tilde y$ has a quality of at least $q_0$.
    This corresponds exactly to a random walk over the vertices with quality at least $q_0$, which is the definition of the $q_0$-quality $x$-prompt graph representation $\graph^{\geq q_0}_x$ of $\perturbationOracle$.

  In addition, the following conditions hold:
    \begin{enumerate}
        \item By leveraging~\Cref{itm:graph-assumption} of~\Cref{thm:impossibility-perturbation}, for every $q \in [q_{\sf min},1]$, the $q$-quality $x$-prompt graph representation $\graph^{\geq q}_x$ is irreducible and aperiodic. Thus, a random walk over the weighted directed graph $\graph^{\geq q_0}_x$ will {\em eventually} converge to its unique stationary distribution $\vector{\pi}_{x,q_0}$ 
        (recall that $q_0 \geq q_{\sf min}$ by assumption).
        
        \item By assumption $\adversary^{\perturbationOracle(\cdot,\cdot),\quality(\cdot,\cdot)}(x,y) \neq \bot$. Thus, $\ctr \geq t-t_{\sf err} = \max_{x\in\cX, q \in [q_{\sf min},1]}\{t_{x,q}\}$ (as defined in~\Cref{itm:graph-assumption} of~\Cref{thm:impossibility-perturbation}) which, in turn, implies $\ctr \geq t_{x,q_0}$ since $q_0 \geq q_{\sf min}$.
        Note that $\ctr$ corresponds to the number of steps performed by the adversary during the random walk over the $q_0$-quality $x$-prompt graph $\graph^{\geq q_0}_x = (\cV^{\geq q_0}_x,\cE^{\geq q_0}_x)$.
    \end{enumerate}
    By leveraging the above two conditions, we conclude that $(i)$ a random walk over $\graph^{\geq q_0}_x = (\cV^{\geq q_0}_x,\cE^{\geq q_0}_x)$ converges to its unique stationary distribution $\vector{\pi}_{x,q_0}$, and $(ii)$ $\adversary^{\perturbationOracle(\cdot,\cdot),\quality(\cdot,\cdot)}(x,y)$'s random walk is composed of at least $\ctr\geq t_{x,q_0}$ steps where $t_{x,q_0}$ is asymptotically larger than the $\epsilon_{\sf dist}$-mixing time of the transition matrix $\matrix{P}_{x,q_0}$ of $\graph^{\geq q_0}_x$ (see~\Cref{itm:graph-assumption} of~\Cref{thm:impossibility-perturbation}).
    Thus, we conclude that the output distribution of $\adversary^{\perturbationOracle(\cdot,\cdot),\quality(\cdot,\cdot)}(x,y)$ is $\epsilon_{\sf dist}$-close to $\vector{\pi}_{x,q_0}$, i.e., 
    \[
        \left\vert \adversary^{\perturbationOracle(\cdot,\cdot),\quality(\cdot,\cdot)}(x,y) - \vector{\pi}_{x,q_0} \right\vert \leq \epsilon_{\sf dist}.
    \]
    This concludes the proof.
\end{proof}

By leveraging~\Cref{lmm:high-quality-watermark-output,lmm:binomial,lmm:stationary-distribution}, we have that for every $\model \in \cM$, for every prompt $x\in\cX$, the following conditions hold:
\begin{enumerate}
    \item Let $\rv{E}$ be the event that $\quality(x,y) = q_0 \geq q_{\sf min}$ and $\adversary^{\perturbationOracle(\cdot,\cdot),\quality(\cdot,\cdot)}(x,y) \neq \bot$.
    Then, by leveraging~\Cref{lmm:high-quality-watermark-output,lmm:binomial} we conclude that $\rv{E}$ occurs with probability at least 
    \[
        \left(1-\frac{v}{100}\right)\left(1 - \sum^{t - t_{\sf err}-1}_{k = 0} {t\choose k} \left(\epsilon_{\sf pert}\right)^{k}\left(1-\epsilon_{\sf pert}\right)^{t-k}\right),
    \]
    where the probability is taken over $(\k,\model)$ output by $\watermark(\model)$, the random coins of the watermarked model $\model_\k$, and the perturbation oracle $\perturbationOracle$.

    \item Conditioned on $\rv{E}$, the quality $\quality(x,y_t)$ of $y_t \neq \bot$ (output by the adversary) is at least $q_0 \geq q_{\sf min}$. This is because $y_t$ is the result of a random walk over the graph $\graph^{\geq q_0}_x = (\cV^{\geq q_0}_x,\cE^{\geq q_0}_x)$ (of the perturbation oracle $\perturbationOracle$) composed of all vertices of quality at least $q_0$ (see also the proof of~\Cref{lmm:stationary-distribution}).

    \item Conditioned on $\rv{E}$, the output $y_t \neq \bot$ (produced by the adversary) is such that $\detect_\k(x,y_t) = 0$ (i.e., $y_t$ is not watermarked) with probability at least $(1-\epsilon_{\sf pos})(1-\epsilon_{\sf dist})$.
    This follows by observing that, conditioned on $\rv{E}$ (i.e., $\quality(x,y) = q_0 \geq q_{\sf min}$ and $\adversary^{\perturbationOracle(\cdot,\cdot),\quality(\cdot,\cdot)}(x,y) \neq \bot$), the output distribution of $\adversary^{\perturbationOracle(\cdot,\cdot),\quality(\cdot,\cdot)}(x,y)$ is $\epsilon_{\sf dist}$-close to the unique stationary distribution $\vector{\pi}_{x,q_0}$ (as defined in~\Cref{lmm:stationary-distribution}).
    In turn, since $\vector{\pi}_{x,q_0}$ is independent of $(\k,\model_\k) \getsr \watermark(\model)$,\footnote{This is because the output distribution of the perturbation oracle $\perturbationOracle$ (which, in turn, defines its corresponding graph) is defined before the sampling of $(\k,\model_\k)$ according to $\watermark(\model)$.} we have that $\detect_\k(x,y_t)=0$ with probability at least $(1-\epsilon_{\sf pos})$ due to the false positive $\epsilon_{\sf pos}$-rate (\Cref{def:false-positive-negative}) of the watermarking scheme $\Pi$.
\end{enumerate}
By combining the above arguments, we conclude that, for every model $\model \in \cM$, for every prompt $x\in\cX$, the following probability holds:
\[
    \prob{
        \begin{array}{c}
            \detect_\k(x,y_t) = 0 \text{ and } \quality(x,y_t) \geq \quality(x,y)
        \end{array}
        :
        \begin{array}{c}
            y \getsr \model_\k(x),\ y_t \getsr \adversary^{\perturbationOracle(\cdot,\cdot),\quality(\cdot,\cdot)}(x,y)
        \end{array}
    } \geq \epsilon 
\]
where $(\k,\model_\k) \getsr \watermark(\model)$ and $\epsilon$ as defined in~\Cref{thm:impossibility-perturbation}.
This concludes the proof of~\Cref{thm:impossibility-perturbation}.

\section{Additional Experimental Results}

\subsection{Qualitative Results through Steps}
\label{app:intermediate}
Given a C4 news article prefix ``\textit{``Whoever gets him, they'll be getting a good one,'' David Montgomery said. INDIANAPOLIS — Hakeem Butler has been surrounded by some of the best wide receivers on the planet this week at the}'', we showcase in~\Cref{fig:intermediate} the results of removing the watermark in the following response:
{\itshape``Whoever gets him, they'll be getting a good one,'' David Montgomery said. INDIANAPOLIS — Hakeem Butler has been surrounded by some of the best wide receivers on the planet this week at the NFL Scouting Combine. They talk about technique, about routes, about film breakdowns. They discuss ideas for how to improve their games. Butler is the outsider. He is the kid who always felt he belonged, but never quite fit in. He is the kid who refused to sit on the bench, even if he didn't know if he'd ever get in a game. He is the kid who traveled from the middle of Iowa on a bus, all the way to Indianapolis this week, so he could show NFL scouts he belonged. He is the kid, who when asked to sit down, he instead stood up. ``I'm not a benchwarmer,'' Butler told The Athletic, ``and it was like that all throughout college.'' Butler is the guy who's not afraid to speak his mind.''
}

Note that for the initial watermarked response, the detection algorithm of KGW gives a z-score of 7.340 and a p-value of 0.000. 
We report the results every $5$ valid random walks.  %
At each generation, we highlighted the different parts of the texts that are modified with the corresponding detection (z-score, p-value) and GPT-4 quality judgment results.

We observe that the detection performance generally keeps decreasing while the new texts are of high quality according to the quality oracle implemented as GPT-4. 
Especially, the oracle score stays at 0 for the final several examples, showing that they are of similar quality to the watermarked response.

\section{Experimental Details}

\subsection{Implementation Details of Attack }\label{app:attack_details}
We discuss our key design choices and implementations. In general, we found our attack effective for all the settings considered and is not susceptible to the hyper-parameters and choices below. 

\textbf{Perturbation oracle.} Recall that we generate watermarked texts with maximum generation tokens of 200 or 512 and attack by replacing one span of the text at a time, thus we set the span length to be 6 for all the attacks in the main table. 
For each infill, we do top-p sampling and p, the minimum. and maximum infilled text length according to~\Cref{tab:hyperparam}. 
Note that we generally use the default hyper-parameters and don't tune them too much.
We incorporate backtracking into the random walk as another error-reduction mechanism. If the perturbation oracle suggests many candidates (above some ``patience'' threshold) without any of them passing the quality checks, then we undo the most recent step of the walk.

\textbf{Quality oracle.}
\label{app:reason}
Note that in general all kinds of watermarks are removable by omitting contents but would degrade the quality. 
Therefore, it is important to ensure the quality at each step of our iterative attack process. 
To realize this goal, we implement three alternative quality oracles, trading off quality, cost, and efficiency (\Cref{tab:qual_tradeoff}). 
In principle, attackers can tailor quality oracles according to their needs, considering the trade-offs among quality, efficiency, and cost.
For example, i) Efficiency and Cost: malicious users can efficiently generate many high-quality texts using GPT-4 for automated phishing and strip the watermark using reward models as the quality oracle with slight degradation in quality. ii) Quality: a student can wait a week to generate a solution with GPT-4 to one assignment problem whose deadline is one week from now with hundreds of dollars for paying for GPT-4 as a quality oracle.

Though our impossibility results are generic and the assumptions can be made stronger when models become more capable, we instantiate our quality oracle reward models + GPT-3.5 for quality checking for most experiments.
The reasons are twofold: firstly, it's much more efficient and less expensive to use reward models for comparing responses to filter obviously bad texts; moreover, we find that GPT-3.5 and GPT-4 have significant position biases \citep{zheng2023judging} that grade the first response with higher quality when evaluating two responses to a given query (prompting details are in~\Cref{app:prompt}). 
We find that such position bias limitation is substantial in practice when the response length is greater than $200$ so it'd be hard to get non-watermarked examples in a reasonable amount of queries even if powerful models can evaluate multiple nuance aspects of individual responses. 
We detail the trade-offs in~\Cref{tab:qual_tradeoff}.

\textbf{Design choices to ensure text quality.}
When masking a span, we split and mask the words rather than tokens to avoid generating nonsensical words that degrade text quality.
To alleviate the impact of position bias on the quality oracle, we query it twice and categorize the results as win, tie, or lose. 
Our goal is to get a new text we non-degrading quality, so we reject the new text if it loses in both rounds at each step.

We filtered out low-quality watermarked examples (e.g. those with a great number of repetitions \citep{zhao2023provable}) since our perturbation oracle may persist in that repetition and the original text would not pass quality oracle in the first place. This can be a reasonable intervention as repetition violates our usable preassumption -  we expect capable models like GPT-4 and future LMs not to produce repetitiveness after watermarking their outputs.

\textbf{Stopping conditions.} 
Users are allowed to design the stopping condition according to their needs and understanding of the detection: 
For the three watermarks we considered, we record and score each intermediate example and early stop on the one with z-score less than 1.645, which is practical whenever users roughly know the detection z-score threshold to make the stopping iteration a tunable hyper-parameter; 
For high-stake settings where users don't want to be detected as using AI-generated texts when users don't have any knowledge about the watermark scheme and detector or when the watermark is not robust enough (for long texts etc such as EXP \citep{kuditipudi2023robust}, Unigram \citep{zhao2023provable}), we stop when at least $\alpha$ of the words are replaced, where $\alpha$ is set to be $70\%$ for texts of length 500 or 512, and proportionally for other lengths, e.g. $80\%$ for length 600.

\newcommand{\heavyminus}{\rule[0.5ex]{1em}{1.8pt}}
\begin{table*}[ht]
\centering
\caption{The trade-offs of different quality oracle instantiation.}
\label{tab:qual_tradeoff}
\vspace{0.1in}
\resizebox{0.57\textwidth}{!}{
\begin{tabular}{cccc}
\toprule 
\toprule
Model & Quality & (API) Costs & Efficiency \\ 
\hline
 Reward Model &  \heavyminus & \coloright & \coloright \\ \hline
 Reward Model + API Error Checking & \coloright & \heavyminus & \coloright \\ \hline
 GPT-3.5/4 API & \coloright & \wrong & \wrong \\ 
\bottomrule
\bottomrule  
\end{tabular}
}
\end{table*}

\subsection{Watermark Details}
\label{app:baselines}
Denote $|x|_G$ as the number of green list tokens for a generated text with length $T$. We experiment with three popular watermark schemes with their default hyper-parameters in general (\Cref{tab:hyperparam}). 
\begin{itemize}
    \item KGW~\citep{kirchenbauer2023watermark} is about selecting a randomized set of ``green'' tokens before a word is generated, and then softly promoting the use of green tokens during sampling, which can be detected efficiently. 
    We adopt a \textit{one proportion z-test}, where $z=2\left(|x|_G-T / 2\right) / \sqrt{T}$. to evaluate the null hypothesis $H_0$:\textit{The text sequence is generated with no knowledge of the red list rule} and choose to reject the null hypothesis if $z > 4$.
    \item  EXP~\cite{kuditipudi2023robust} is a distortion-free watermark framework that preserves the original LM’s text distribution, at least up to some maximum number of generated tokens. 
    For detection, we compute a p-value with respect to a test statistic that measures the minimum cost alignment between length $k$ subsequences of the text and key, via a permutation test with $5000$ resamples. 
    If $\phi$ returns a small p-value then the text is likely watermarked. %
    \item Unigram \citep{zhao2023provable} is proposed as a watermark robust to edit property. 
    We calculate the number of green list tokens $|x|_G$ as well as the z-statistic $z=\left(|x|_G-\gamma T\right) / \sqrt{T \gamma(1-\gamma)}$ where $\gamma=0.5$ means the fraction of the vocabulary included in the green list. The text is predicted as AI-generated if $z > 6$. 
    We set the strength parameter $\delta=2$, the larger $\delta$ is, the lower the quality of the watermarked LM, but the easier it is to detect.
    \item Stable Signature \citep{fernandez2023stable} refers to a method of embedding invisible watermarks into images generated by Latent Diffusion Models (LDMs).
    This approach involves fine-tuning the latent decoder part of the image generator, conditioning it on a binary signature. 
    The modified decoder then generates images that inherently contain this watermark. A pre-trained watermark extractor can later retrieve the hidden signature from any image produced by this model, allowing for the identification of the image's origins even after substantial modifications.
    We utilize the existing VAE checkpoint from \href{https://huggingface.co/imatag/stable-signature-bzh-sdxl-vae-strong}{IMATAG} and use its default hyper-parameters.
    \item Invisible watermark \citep{shield2023ivw} is a default (classic) watermark to the Stable Diffusion model series, which utilizes frequency space transformations to embed watermarks invisibly into images, using Discrete Wavelet Transform and Discrete Cosine Transform. These methods embed watermark bits into the significant frequency components of an image, ensuring robustness against alterations like noise and compression while being sensitive to size and aspect ratio changes. The process involves converting the image from BGR to YUV color space, applying DWT to isolate frequency components, and then using DCT to embed the watermark, making it imperceptible but extractable with appropriate algorithms.
\end{itemize}

\begin{table*}[ht]
\centering
\caption{Default hyperparameters of our attack for LM watermarks. } \label{tab:hyperparam}
\adjustbox{max width=\textwidth}{
\begin{tabular}{c c c c}
    \toprule
    \toprule
     & KGW & EXP & Unigram \\
    \hline
    Attack steps & 200 & 300 & 300 \\ 
    Secret key & 15485863 & 42 & 0 \\
    z stopping threshold & \multicolumn{3}{c}{1.645} \\
    Max watermarked length & \multicolumn{3}{c}{\{200, 512\}} \\
    top-p of $P$ & \multicolumn{3}{c}{0.95} \\
    Span length & \multicolumn{3}{c}{\{4, 6, 8\}} \\
    Num of spans $l$ & \multicolumn{3}{c}{1} \\
    Min infill length & \multicolumn{3}{c}{\{$l$, $1.5l$\}} \\
    Max infill length & \multicolumn{3}{c}{\{$1.5l$, $2l$\}} \\  
    \bottomrule
    \bottomrule 
\end{tabular}
}
\end{table*}

\begin{table*}[ht]
\centering
\caption{Default hyperparameters of our attack for VLM watermarks. } \label{tab:hyperparam}
\adjustbox{max width=\textwidth}{
\begin{tabular}{c c c c}
    \toprule
    \toprule
     & Stable Signature & Invisible Watermark \\
    \hline
    & \multicolumn{2}{c}{Watermarked Model} \\ 
    \cline{2-3}
    Guidance scale & 0 & 7.5 \\
    Num of inference steps & 4 & 50 \\
    Secret key & \multicolumn{2}{c}{0}  \\
    Prompts & \multicolumn{2}{c}{\href{https://huggingface.co/datasets/Gustavosta/Stable-Diffusion-Prompts}{Gustavosta/Stable-Diffusion-Prompts}} \\
    Watermark strength & Strong & - \\
    Scheduler & \multicolumn{2}{c}{KarrasDiffusionSchedulers} \\
    \cline{2-3}
    & \multicolumn{2}{c}{Attack} \\ 
    \cline{2-3}
    Attack steps & \multicolumn{2}{c}{100} \\
    Square Mask ratio & \multicolumn{2}{c}{0.02} \\
    Guidance scale & \multicolumn{2}{c}{5} \\
    Num of inference steps & \multicolumn{2}{c}{100} \\
    Scheduler & \multicolumn{2}{c}{LMSDiscreteScheduler} \\ 
    \bottomrule
    \bottomrule 
\end{tabular}
}
\end{table*}

\subsection{Prompt Examples}
\label{app:prompt}
Recall that it is a design choice to implement the quality oracle as a combination of a reward model and GPT-3.5/GPT-4 to make sure the newly generated sample satisfies certain rubrics by rejecting samples with errors. In our case, we use the prompt in~\Cref{fig:error_prompt}.

\begin{figure*}[h]
    \centering %
    \begin{tcolorbox}[
    title={\textbf{Prompt for avoiding basic errors that can degrade text quality}},
    colframe=cornflowerblue,
    colback=white]
    Below are two candidate responses to the query \{query Q\}:
    
    Original response:
    \{response A\}
    
    New response:
    \{response B\}
    
    Text quality is affected by factors such as unnecessary repetitions, grammar, coherence, relevance, and accuracy of the responses.
    Especially, having grammatical errors, repetitiveness, capitalization errors or punctuation mistakes would greatly degrade the quality of a response.
    
    Therefore, is the new modified response of equal or higher quality compared to the original response? If so, answer Yes, otherwise answer No.

    \end{tcolorbox}
    \caption{GPT-3.5 prompt for avoiding basic errors that can degrade text quality. \label{fig:error_prompt}}
\end{figure*}

We showcase our hand-crafted prompts for our GPT-3.5 or GPT-4 quality oracle for final evaluation (\Cref{fig:oracle_5choice_prompt}). 
We also experiment with the system prompt in \citep{zheng2023judging}, hoping to reduce biases regarding position, length etc but find that it has very minor effects on the results.

\begin{figure*}[h]
    \centering %
    \begin{tcolorbox}[
    title={\textbf{Prompt for evaluating and comparing text quality with five choices.}},
    colframe=cornflowerblue,
    colback=white]

    Below are two candidate responses to the query \{query Q\}:
    
    Response A: \{response A\}
    
    Response B: \{response B\}
    
    Compare which of the two responses above is a better response to the given query. Explain your reasoning step by step.

      (1) Response A is much better than response B
      
      (2) Response A is slightly better than response B
      
      (3) Responses A and B have similar quality
      
      (4) Response B is slightly better than response A
      
      (5) Response B is much better than response A
    \end{tcolorbox}
    \caption{GPT-4 prompt for evaluating and comparing text quality with five choices \label{fig:oracle_5choice_prompt}}
\end{figure*}

\begin{figure}
    \centering
    \caption{Intermediate text after attack (left, \textcolor{candypink}{red}) and its original watermarked text (right, \textcolor{caribbeangreen}{green})}
    \label{fig:intermediate}
    \begin{subfigure}[b]{\textwidth}
        \centering
        \includegraphics[width=\textwidth]{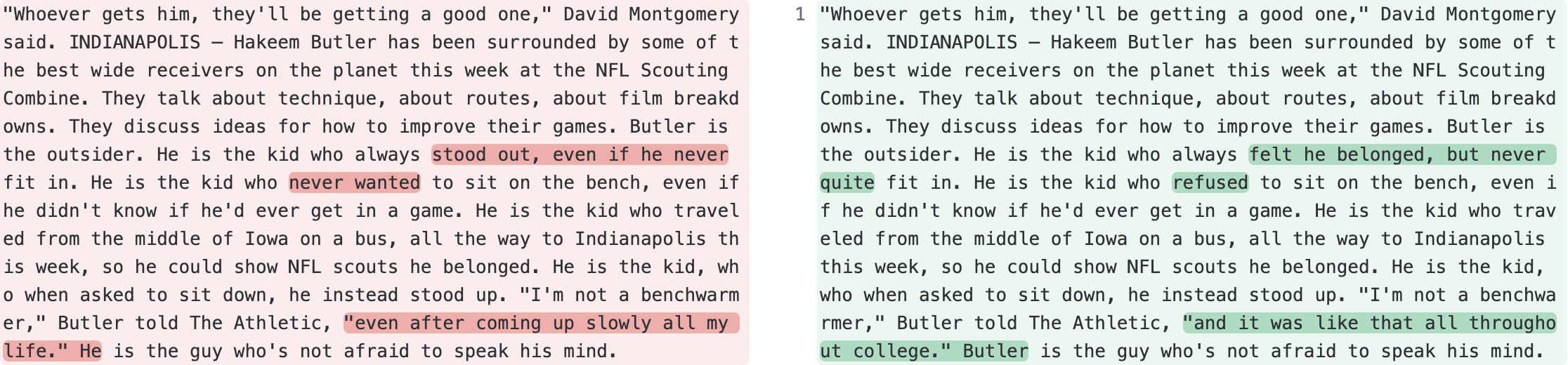}\\
        \caption{The 5-th step, z-score = 5.35, p-value = 0.00, GPT-4 quality oracle score = 0.00. }
        \label{fig:qual1}
    \end{subfigure}
\end{figure}

\begin{figure}
    \centering
    \ContinuedFloat
    \begin{subfigure}[b]{\textwidth}
        \centering
        \includegraphics[width=\textwidth]{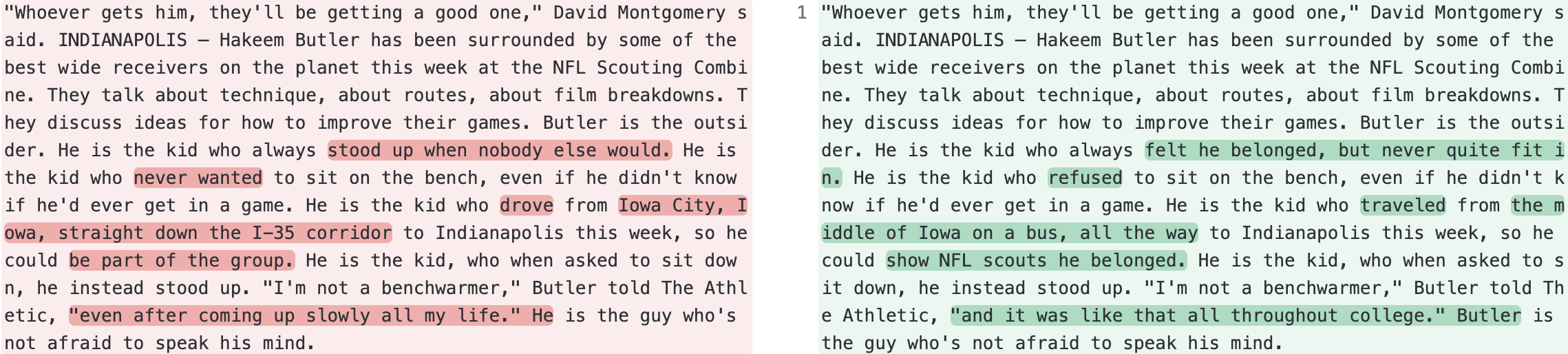}\\
        \caption{The 10-th step, z-score = 4.59, p-value = 0.00, GPT-4 quality oracle score = 0.00. }
        \label{fig:qual2} 
    \end{subfigure}
    \label{fig:intermediate}
\end{figure}

\begin{figure}
    \centering
    \ContinuedFloat
    \begin{subfigure}[b]{\textwidth}
        \centering
        \includegraphics[width=\textwidth]{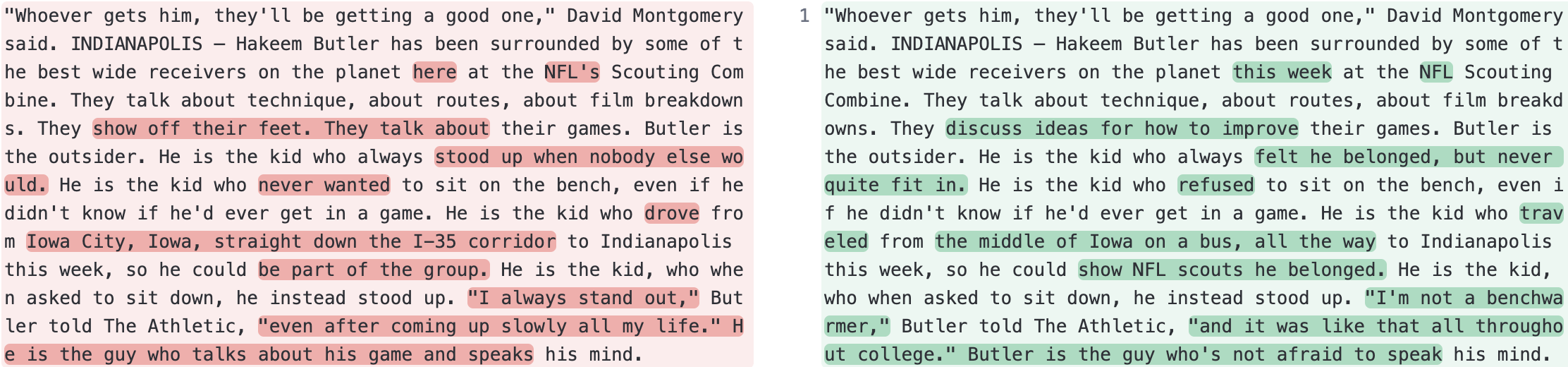}\\
        \caption{The 15-th step, z-score = 2.78, p-value = 0.0027, GPT-4 quality oracle score = 0.00. }
        \label{fig:qual3}
    \end{subfigure}
    \label{fig:intermediate}
\end{figure}

\begin{figure}
    \centering
    \ContinuedFloat    
    \begin{subfigure}[b]{\textwidth}
        \centering
        \includegraphics[width=\textwidth]{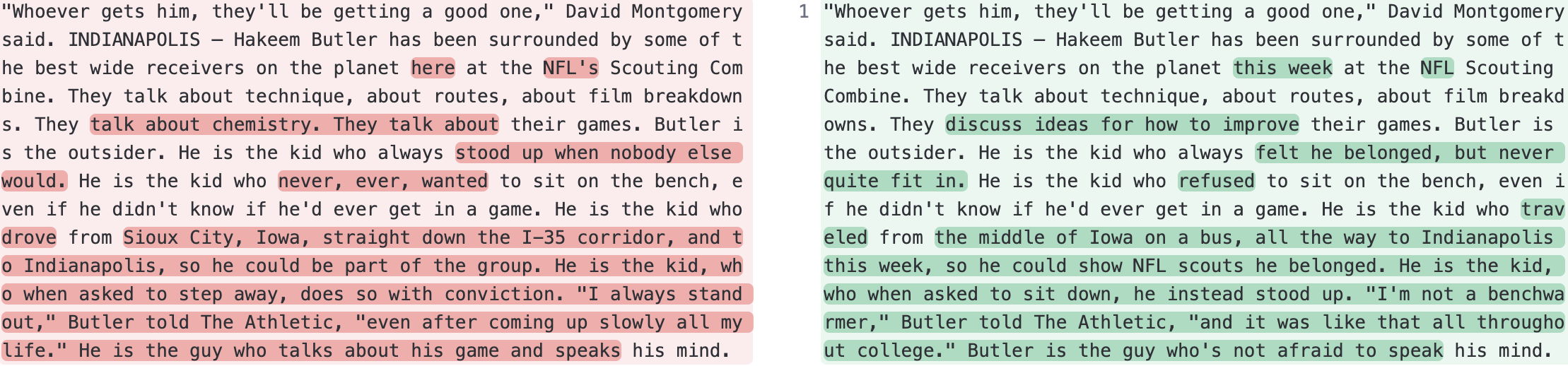}\\
        \caption{The 20-th step, z-score = 1.66, p-value = 0.048, GPT-4 quality oracle score = 0.00.}
        \label{fig:qual4} 
    \end{subfigure}
    \label{fig:intermediate}
\end{figure}

\begin{figure}
    \centering
    \ContinuedFloat 
    \begin{subfigure}[b]{\textwidth}
        \centering
        \includegraphics[width=\textwidth]{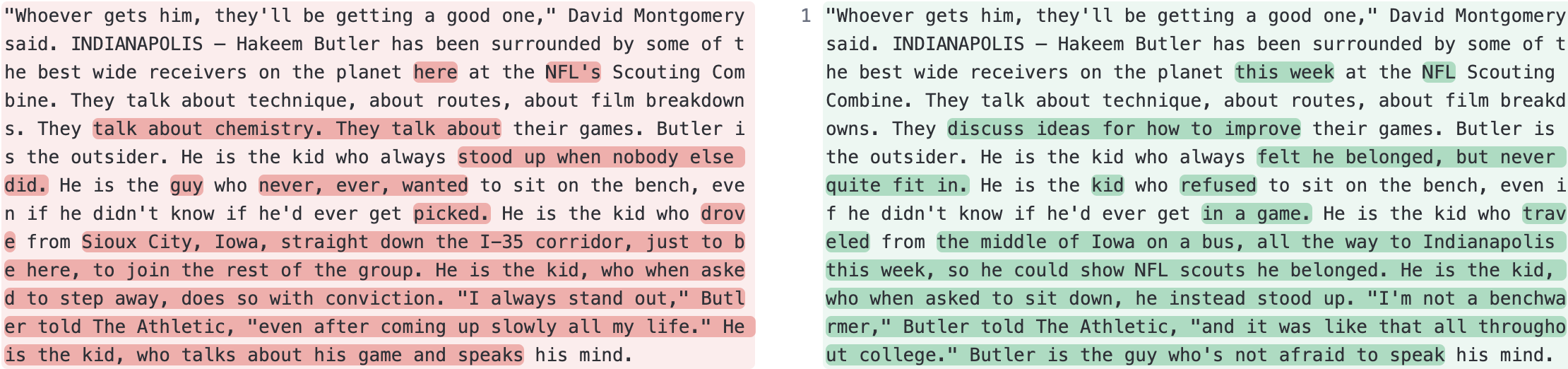}\\
        \caption{The 25-th step, z-score = 1.29, p-value = 0.0993, GPT-4 quality oracle score = 0.0}
        \label{fig:qual5}
    \end{subfigure}
    \label{fig:intermediate}
\end{figure}

\begin{figure}
    \centering
    \ContinuedFloat 
    \begin{subfigure}[b]{\textwidth}
        \centering
        \includegraphics[width=\textwidth]{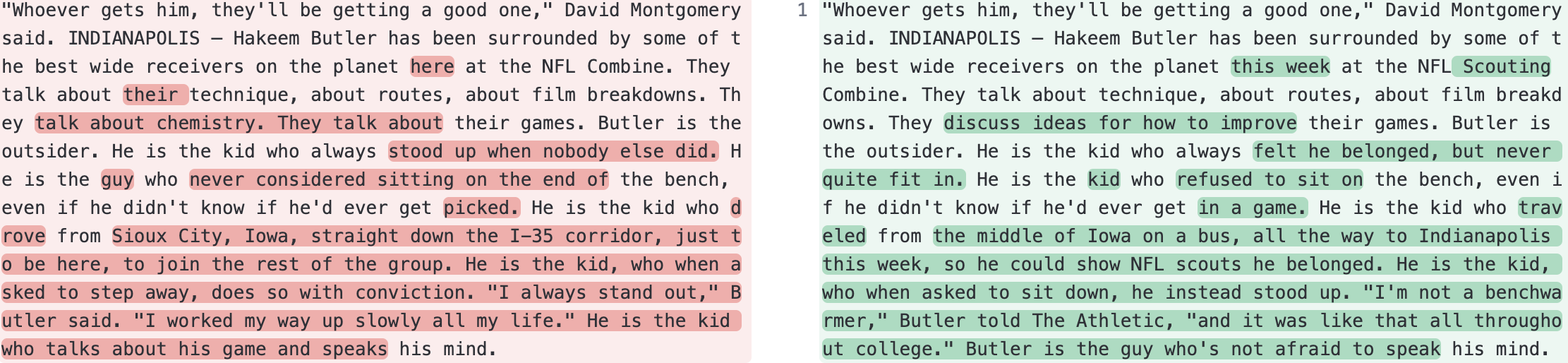}\\
        \caption{The 30-th step, z-score = 0.654, p-value = 0.257, GPT-4 quality oracle score = 0.00.}
        \label{fig:qual6} 
    \end{subfigure}
    \label{fig:intermediate}
\end{figure}

\begin{figure}
    \centering
    \ContinuedFloat 
    \begin{subfigure}[b]{\textwidth}
        \centering
        \includegraphics[width=\textwidth]{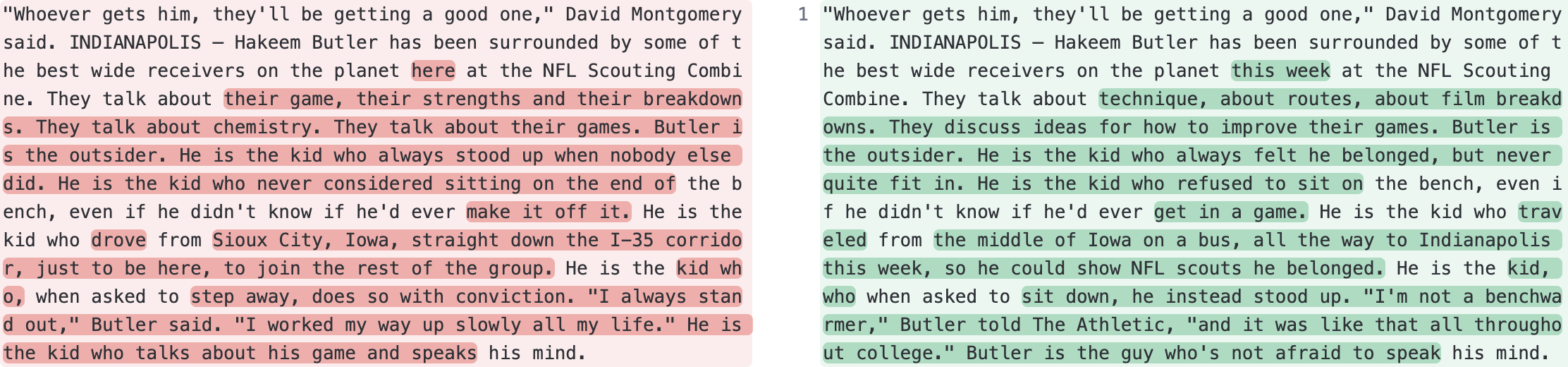}\\
        \caption{The 35-th step, z-score = 0.13, p-value = 0.448, GPT-4 quality oracle score = 0.00.}
        \label{fig:qual7}
    \end{subfigure}
    \label{fig:intermediate}
\end{figure}

\begin{figure}
    \centering
    \ContinuedFloat 
    \begin{subfigure}[b]{\textwidth}
        \centering
        \includegraphics[width=\textwidth]{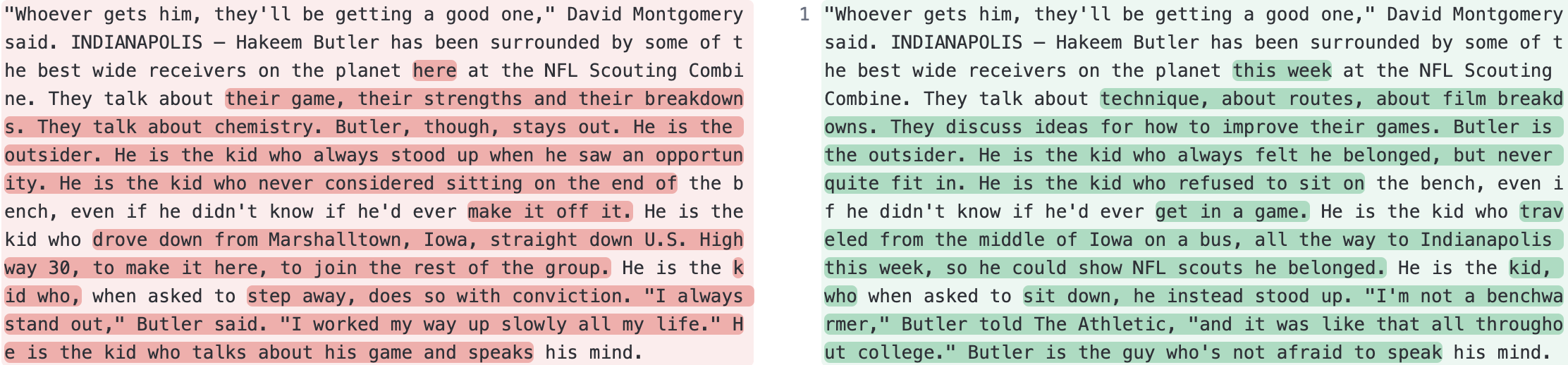}\\
        \caption{The 40-th step, z-score = 0.52, p-value = 0.303, GPT-4 quality oracle score = 0.00.}
        \label{fig:qual8}
    \end{subfigure}
    \label{fig:intermediate}
\end{figure}

\begin{figure}
    \centering
    \ContinuedFloat 
    \begin{subfigure}[b]{\textwidth}
        \centering
        \includegraphics[width=\textwidth]{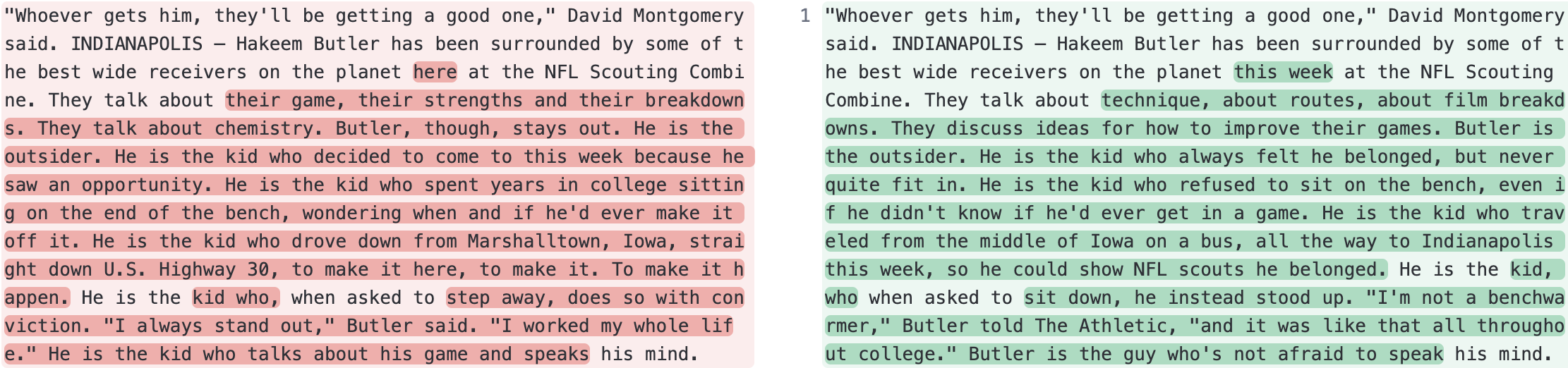}\\
        \caption{The 45-th step, z-score = -0.26, p-value = 0.602, GPT-4 quality oracle score = 0.00.}
        \label{fig:qual9}
    \end{subfigure}
    \label{fig:intermediate}
\end{figure}

\begin{figure}
    \centering
    \ContinuedFloat 
    \begin{subfigure}[b]{\textwidth}
        \centering
        \includegraphics[width=\textwidth]{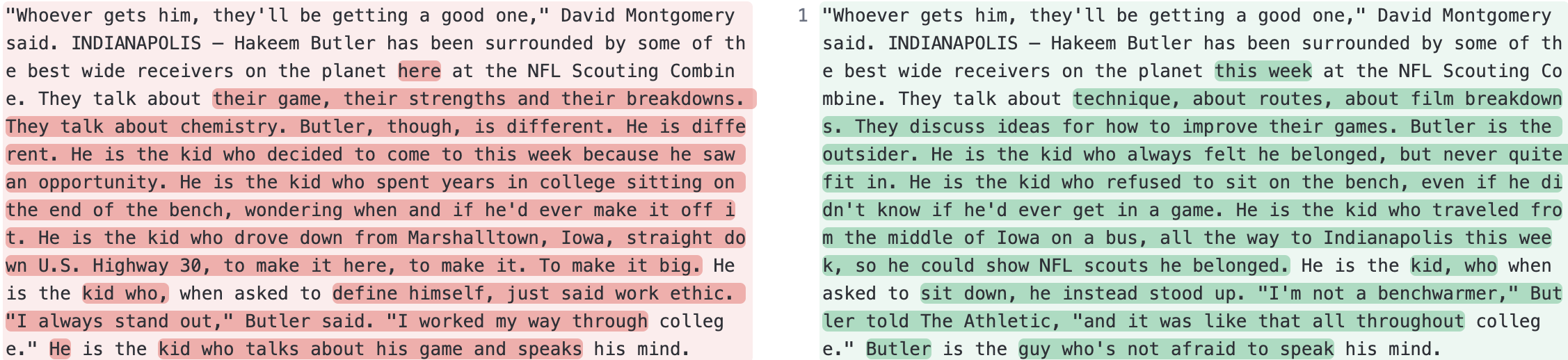}\\
        \caption{The 50-th step, z-score = 0.52, p-value = 0.30, GPT-4 quality oracle score = 0.00.}
        \label{fig:qual10}
    \end{subfigure}
    \label{fig:intermediate}
\end{figure}

\end{document}